\documentclass[a4paper,fleqn]{cas-sc}

\usepackage{bm}
\usepackage{array} 
\usepackage{tabularx} 
\usepackage{hyperref}
\usepackage{amsthm}
\usepackage{enumitem}
\usepackage{amsfonts}
\usepackage{amssymb}
\usepackage{booktabs}
\usepackage{colortbl}
\usepackage{graphicx}
\usepackage{caption}
\usepackage{subcaption}
\usepackage{multicol}
\usepackage{rotating}
\usepackage{tikz-cd}
\usepackage{multirow}
\usepackage[ruled,vlined,linesnumbered]{algorithm2e}

\usepackage[authoryear]{natbib}

\def\tsc#1{\csdef{#1}{\textsc{\lowercase{#1}}\xspace}}
\tsc{WGM}
\tsc{QE}
\tsc{EP}
\tsc{PMS}
\tsc{BEC}
\tsc{DE}

\newtheorem{theorem}{Theorem}
\newtheorem{definition}{Definition}[section] 

\begin{document}

\let\WriteBookmarks\relax
\def\floatpagepagefraction{1}
\def\textpagefraction{.001}
\let\printorcid\relax
\shorttitle{}

\title [mode = title]{Toward Fair Graph Neural Networks Via Dual-Teacher Knowledge Distillation}                      

\author[1]{Chengyu Li}[style=chinese]
\author[2]{Debo Cheng}[style=chinese]
\ead{chengdb2016@gmail.com}
\cormark[1]
\author[3]{Guixian Zhang}[style=chinese]

\author[1]{Yi Li}[style=chinese]

\author[1]{Shichao Zhang}[style=chinese]

\address[1]{School of Computer Science and Engineering; Guangxi Key Lab of MSMI; MOE Key Lab of EBIT, Guangxi Normal University, Guilin, 541004, Guangxi, China}

\address[2]{{School of Computer Science and Technology},
            {Hainan University}, 
            {Haikou},
            {Hainan},
            {570228},
            {China }}
\address[3]{School of Computer Science and Technology, China University of Mining and Technology, Xuzhou, Jiangsu, 221116, China}

\cortext[1]{Corresponding author}

\begin{abstract}
Graph Neural Networks (GNNs) have demonstrated strong performance in graph representation learning across various real-world applications. However, they often produce biased predictions caused by sensitive attributes, such as religion or gender, an issue that has been largely overlooked in existing methods. Recently, numerous studies have focused on reducing biases in GNNs. However, these approaches often rely on training with partial data (e.g., using either node features or graph structure alone), which can enhance fairness but frequently compromises model utility due to the limited utilization of available graph information.
To address this trade-off, we propose an effective strategy to balance fairness and utility in knowledge distillation. Specifically, we introduce FairDTD, a novel \underline{Fair} representation learning framework built on \underline{D}ual-\underline{T}eacher \underline{D}istillation, leveraging a causal graph model to guide and optimize the design of the distillation process. In particular, FairDTD employs two fairness-oriented teacher models: a feature teacher and a structure teacher, to facilitate dual distillation, with the student model learning fairness knowledge from the teachers while also leveraging full data to mitigate utility loss. To enhance information transfer, we incorporate graph-level distillation to provide an indirect supplement of graph information during training, as well as a node-specific temperature module to improve the comprehensive transfer of fair knowledge.
Experiments on diverse benchmark datasets demonstrate that FairDTD achieves optimal fairness while preserving high model utility, showcasing its effectiveness in fair representation learning for GNNs.
\end{abstract}



\maketitle

\section{Introduction}
Graph Neural Networks (GNNs) have garnered widespread attention in recent years due to their exceptional ability to model non-Euclidean structured data, such as graphs, achieving remarkable success across various domains. For instance, in drug discovery, GNNs effectively capture the complex interactions between atoms and chemical bonds within molecular graphs, thereby improving the accuracy of molecular property prediction and drug screening \citep{dai2024comprehensive}. In recommendation systems, GNNs model the interaction graph between users and items to uncover latent collaborative filtering patterns, significantly enhancing recommendation performance \citep{ma2024cross}. Unlike traditional machine learning models, the key advantage of GNNs lies in their structure-aware capability: through iterative neighborhood feature aggregation, GNNs enable context-aware node representations, making them highly expressive and adaptable in realworld tasks where graph-structured data is prevalent.

However, due to the neighborhood aggregation mechanism inherent in GNNs, the model may inherit and even amplify underlying social biases present in graph data, leading to discriminatory predictions influenced by sensitive attributes such as gender and race \citep{zou2025graph, li2024contrastive}. In applications involving high-risk decision-making, such as fraud detection \citep{li2022devil} and credit scoring \cite{yeh2009comparisons}, these biased predictions can exacerbate social inequities \citep{mehrabi2021survey}. Additionally, GNNs model on non-independent and identically distributed graph-structured data, while most traditional fair machine learning methods rely on the independent and identically distributed assumption. This makes it difficult for them to be directly applied to graph neural networks. Therefore, achieving fairness within the GNNs framework is of higher complexity and challenge. To address this, researchers have proposed fair graph representation learning. Existing research \citep{zhang2024disentangled, li2024quantum} attributes bias in GNNs to both node features and graph structure. Through feature propagation, initially unbiased node features can become correlated with sensitive attributes, leading to unintended leakage of sensitive information in node representations \citep{wang2022improving}. Additionally, nodes sharing similar sensitive attributes are frequently more interconnected in graph-structured data \citep{la2010randomization}, causing similarity in representations that can result in prediction bias.

Recent studies have focused on reducing bias in GNNs by enhancing model fairness while preserving utility. A common strategy is to modify the training data to preemptively mitigate bias. For instance, this can involve increasing the representation of disadvantaged groups \citep{current2022fairegm} or updating the adjacency matrix with fairness-based constraints \citep{li2021dyadic}. Such preprocessing steps aim to reduce inherited bias during training. Another approach introduces fairness constraints directly into the GNN’s objective function. For example, FairGAT \citep{kose2024fairgat} employs a bias-mitigating attention mechanism, while FairGNN \citep{dai2021say} employs adversarial learning to generate node representations independent of sensitive attributes. However, these methods face limitations: (1) bias exists in both node features and graph structure, which requires complex model design to handle both sources simultaneously; (2) methods that eliminate sensitive attributes are often highly specific, reducing their generalizability to diverse applications. Note that there is a critical question: can we develop a simpler yet effective approach to improve fairness by addressing biases in both node features and graph structure directly within the training data?

Inspired by reducing the sources of data bias \citep{zhang2022hyper,zhang2003data,bai2021understanding, zhang2021challenges}, it has been confirmed that simply training with partial data (i.e., using only node features or only graph structure) can significantly improve fairness and is readily extensible to other models \citep{zhu2024devil}. However, as we analyzed in Section \ref{analysis}, while partial data training substantially enhances fairness, it often comes at the cost of reduced model utility.

Therefore, the key challenge in training with partial data  lies in effectively balancing fairness and utility. Existing studies have explored different approaches to achieve a better trade-off. For instance, FUGNN \citep{luo2024fugnn}, based on spectral graph theory, addresses this balance by optimizing the feature vector distribution and truncating the spectrum. Similarly, FairSAD \citep{zhu2024fair} balances fairness and utility by disentangling sensitive attributes and employing a channel masking mechanism to separate sensitive information while preserving task-relevant information. However, these methods are primarily designed for training on complete data and fail to address the trade-off between fairness and utility when working with partial data.

In this context, knowledge distillation \citep{hinton2015distilling} offers a novel and promising solution for balancing fairness and utility by passing knowledge from a teacher model to a student model. Building on this idea, we propose a strategy where a student model extracts fairness knowledge from a teacher model trained on partial data while also learning from complete data to maintain utility. However, this strategy faces two challenges: (1) Utility loss in the teacher model: although teacher models trained on partial data can enhance fairness, their utility often declines due to limited data completeness, which may in turn limit its guiding role for the student model; (2) Limited knowledge transfer capability: although teacher models trained on partial data may acquire fairness-related knowledge, it remains unclear whether this knowledge can be transferred to the student model in a stable and effective manner. Moreover, most existing knowledge distillation methods employ a fixed temperature parameter, which fails to account for the varying complexity of prediction across different samples and may lead to imbalanced guidance during the distillation process \citep{wei2024dynamic, li2023curriculum}.

To tackle these challenges, we introduce FairDTD, a novel \underline{Fair} representation learning framework built on \underline{D}ual-\underline{T}eacher \underline{D}istillation. Drawing on a causal modeling perspective, we first conduct a systematic theoretical analysis of fair representation learning from partial data training. To address utility loss, FairDTD introduces two fairness-aware teacher models: a feature teacher and a structure teacher. It enables the student model to learn fairness-aware knowledge from these complementary sources through a dual distillation mechanism. Graph distillation is integrated to compensate for missing structural information during training. To tackle limited knowledge transfer, we design a node-specific temperature module that dynamically adjusts the distillation temperature based on prediction confidence, improving both precision and adaptability in knowledge transfer.

Our contributions can be summarized as follows:
\begin{itemize}
[leftmargin=*,labelindent=1.5em,itemsep=-0.1em]
\item We introduce a causal graph model and conduct both theoretical and empirical analyses of fair representation learning under partial data training.
\item We propose a novel dual-teacher distillation approach (FairDTD) combined with graph distillation and an adaptive temperature mechanism for more effective and fairness-aware knowledge transfer.
\item Experiments on multiple real-world datasets show that FairDTD achieves a better balance between fairness and utility, significantly enhancing fairness without compromising model performance.
\end{itemize}

\section{Related work}
In this section, we discuss relevant studies pertinent to our presented FairDTD framework, focusing on developments in GNNs and methods aimed at promoting fairness within GNNs.

\subsection{Graph neural networks}
In recent years, GNNs have attracted considerable attention owing to their exceptional performance in handling graph-structured data \cite{zhang2024bayesian,deng2023tts,zhang2024mitigating}. To be used in a variety of downstream tasks, including node classification \citep{wang2024heterogeneous}, graph classification \citep{sui2022causal}, and link prediction \citep{baghershahi2023self}, GNNs aim to generate embedding vectors that efficiently integrate both graph structure and node features.

Existing GNN methods primarily focus on spatial-based approaches, which perform message passing directly on the graph’s adjacency structure, aggregating information from neighboring nodes layer by layer. For instance, Graph Convolutional Network (GCN) \citep{kipf2016semi} uses a convolutional operation to aggregate the features of neighboring nodes to update node representations. Graph Isomorphism Network (GIN) \citep{xu2018powerful} introduces a straightforward yet potent message-passing mechanism to ensure that different graph structures yield distinct node representations.

The impressive performance of GNNs has facilitated their extensive adoption in a variety of real-world applications, including key decision-making \citep{fu2025himul} and anomaly detection in online banking transactions \citep{yuan2022explainability}. For instance, HGNRec \citep{li2024homogeneous} employs homogeneous GNNs for third-party library recommendations, LA-MGFM \citep{zhao2023mgfm} uses graph neural networks with multi-graph fusion to predict legal judgments, and MLAGs \citep{li2024multi} leverages multi-view graph convolutional networks to predict loan default risk. Additionally, a category of methods based on GNNs models local statistical channels across multiple geographic grids by integrating heterogeneous Markov graphs with enhanced variational Bayesian inference, achieving efficient, accurate, and highly interpretable wireless channel angle power spectrum estimation \cite{tr2021synchronization, cao2024input}.

However, the success of GNNs has prompted concerns about their potential to propagate and amplify biases in graph-structured data. Studies \citep{dai2021say, kose2021fairness} have highlighted that GNNs may unintentionally produce biased outcomes against certain demographic groups, underscoring the need for designing fairer GNN models for graph-based tasks.

\subsection{Fairness in graph neural networks}
When processing graph-structured data, GNNs are susceptible to biases resulting from data imbalances, just like traditional machine learning models, which might produce discriminatory outcomes \citep{chouldechova2018frontiers}. This has prompted growing interest in fairness within GNNs, often categorized into group fairness \citep{zhang2023fpgnn}, counterfactual fairness \citep{agarwal2021towards}, and individual fairness \citep{dong2021individual}. In this work, we concentrate on group fairness, aiming to ensure that predictions are unbiased across demographic groups defined by sensitive attributes.

The two main methods to improve group fairness in GNNs are pre-processing methods and in-training methods. Pre-processing methods aim to mitigate data bias before training. For example, EDITS \citep{dong2022edits} reduces distributional discrepancies between demographic groups by modifying graph structure and node attributes. FairDrop \citep{spinelli2021fairdrop} applies a fair edge dropout strategy to address structural bias in graphs. SRGNN \citep{zhang2024learning} balances graph structure bias by reducing neighbors of high-degree nodes and increasing connections for low-degree nodes. In-training methods modify the GNN objective function to promote fairness during representation learning. For example, FairVGNN \citep{wang2022improving} generates fair representations through adversarial debiasing. FairSIN \citep{yang2024fairsin} incorporates fairness-promoting features before message passing to enrich final node representations. FairINV \citep{zhu2024one} reduces spurious correlations between sensitive attributes and labels to produce fairer GNN predictions. FG-SMOTE \cite{wang2025fg} introduces a desensitized node representation mechanism and a fairness-aware link predictor, systematically addressing fairness concerns in minority oversampling caused by subgroup under-representation and structural bias. Themis \cite{wang2025towards} leverages a Bayesian variational autoencoder to infer proxy sensitive attributes and incorporates disentangled latent variables alongside a fairness normalization module, enabling fair graph representation learning without requiring on explicit demographic information.

Different from the methods reviewed above, our proposed FairDTD effectively reduces bias by training teacher models on partial data. By integrating this with knowledge distillation, FairDTD alleviates the fairness–utility trade-off commonly encountered in GNNs. This strategy of integrating “partial data training” and “knowledge distillation” provides a new perspective for enhancing the fairness of GNNs, especially demonstrating innovation in directions not deeply explored by previous work. Through the dual-teacher design and dynamic distillation mechanism, FairDTD achieves an effective balance between fairness enhancement and model performance maintenance, further expanding the research paradigm  for fair GNN learning.

Moreover, fair graph representation learning has significant potential in various real-world applications, particularly those involving graph-structured data and requiring fairness-aware automated decision-making. For example, in the financial services sector, including applications, such as credit scoring and fraud detection, fair graph representations help mitigate potential discrimination against sensitive demographic groups (e.g., by gender or race), thereby enhancing fairness, transparency, and interpretability. Similarly, in tasks highly dependent on graph-structured modeling, such as public policy-making, intelligent recruitment systems, and personalized recommendation, incorporating fairness mechanisms helps improve the credibility, reliability, and social sustainability of artificial intelligence systems \cite{zhang2024trustworthy, xia2020part}.

\section{Preliminaries}
In this section, we define the pertinent notations and provide an overview of knowledge distillation. Subsequently, we perform a causal analysis of our research problem using a causal graph and formally define the problem addressed in our work, which serves as the foundation for our FairDTD framework.

\subsection{Notations and problem formulation}
We consider an undirected graph $\bm{\mathcal{G}} = (\bm{\mathcal{V}}, \bm{\mathcal{E}}, \bm{\mathrm{X}})$, where $\bm{\mathcal{V}}$ represents the set of nodes and $\bm{\mathcal{E}}$ represents the set of edges. $|\bm{\mathcal{V}}| = N$ represents the number of nodes. $\bm{\mathrm{X}} \in \mathbb{R}^{N \times F}$ represents the node feature matrix, where $F$ is the dimension of the node feature. $\bm{\mathrm{A}} \in \{0,1\}^{N \times N}$ represents the adjacency matrix, where $\bm{\mathrm{A}}_{ij} = 1$ if and only if $(v_i, v_j) \in \bm{\mathcal{E}}$. Additionally, $S \in \{0,1\}^{N}$ contains the binary-sensitive attribute of each node. 

The majority of existing GNNs operate based on the message-passing method, which aggregates messages from neighboring nodes to update node representations iteratively. Formally, the $l$-th layer of a GNN is defined as:
\begin{equation}
m_v^{(l)}=\operatorname{AGG}^{(l)}\left(\left\{h_u^{(l-1)}: u \in \mathcal{N}(v)\right\}\right),
\end{equation}
\vspace{-1em}
\begin{equation}
h_v^{(l)}=\operatorname{UPD}^{(l)}\left(h_v^{(l-1)}, m_v^{(l)}\right),
\end{equation}
where $l$ is the layer index, $\operatorname{AGG}^{(l)}(\cdot)$ denotes the aggregation function, and $\operatorname{UPD}^{(l)}(\cdot)$ denotes the update function at the $l$-th layer. Here, $\mathcal{N}(v)$ denotes the collection of nodes adjacent to node $v$. The node representation at the final layer is $z_v = h_v^{(L)}$, where $L$ is the network's total number of layers.

In this paper, we aim to address the challenges associated with training on partial data. To this end, we formally define the problem as follows:

{\raggedright\textbf{Problem formulation.}}\hspace{1em}\textit{Given a graph $\bm{\mathcal{G}} = (\bm{\mathcal{V}}, \bm{\mathcal{E}}, \bm{\mathrm{X}})$ and a GNN-based teacher-student framework, which consists of a trained structure teacher $f_{str}$, a trained feature teacher $f_{fea}$, and a student model $f_{stu}$ to be trained, our objective is to maintain the overall utility of the model while developing a fairer student model under the guidance of the two fairness teachers.}

\subsection{Knowledge distillation}
Knowledge distillation seeks to transfer knowledge from a cumbersome, complicated teacher model to a more lightweight student model, enabling the student to achieve performance comparable to the teacher. In this work, we emphasize the transfer of fair and rich knowledge during the distillation process, rather than solely focusing on model compression.

To facilitate knowledge transfer, the knowledge distillation approach was designed by Hinton et al.~\citep{hinton2015distilling} to align the softened outputs generated by both the teacher and student models. Two supervisory signals are utilized in the training of the student model: (1) Hard Labels: The true labels from the training dataset. (2) Soft Labels: The label predictions of the teacher model. Formally, let $\bm{\mathrm{Z}}_{tea}$ represent the output logits from the teacher model, and $\bm{\mathrm{Z}}_{stu}$ denote the output logits from the student model, the loss to be minimized is defined as:
\begin{equation}
\mathcal{L}=\sum_{v \in \mathcal{V}}\mathcal{L}_{CE}\left(y_{v}, 
 (\operatorname{softmax}(z_v^{stu}))\right) + \alpha \mathcal{L}_{KD},
\end{equation}
where $\mathcal{L}_{CE}$ denotes the cross-entropy loss and the hyperparameter $\alpha$ regulates the relative contribution of the two loss functions. The knowledge distillation loss, $\mathcal{L}_{KD}$, is given by:
\begin{equation}
\mathcal{L}_{KD}=\mathcal{H}_{K L}\left(\operatorname{softmax}(\frac{\bm{\mathrm{Z}}_{stu}}{\tau})  \|  
\operatorname{softmax}(\frac{\bm{\mathrm{Z}}_{tea}}{\tau})\right),
\end{equation}
where $\mathcal{H}_{K L}$ is the KL-divergence, and $\tau$ is a temperature hyperparameter that regulates the smoothness of the two logit distributions by appropriately scaling them. A higher temperature makes the distribution flatter, generating softer predictions; a lower temperature enlarges the difference between the two distributions, generating harder predictions.

\subsection{Analysis of partial data training}\label{analysis}
In this subsection, we present the motivation behind partial data training from both theoretical and experimental perspectives. Recent studies have demonstrated that incorporating causal learning techniques into GNNs can more effectively address trustworthiness concerns by capturing the underlying causal relationships within the data~\citep{lin2024learning, fan2022debiasing}. To this end, we employ a causal graph~\citep{pearl2009causality, ling2024fair} to represent the causal relationships and analyze the sources of bias, illustrating how the sensitive attribute influences predictions through both node features and graph structure. Guided by this analysis, we further investigate the task of partial data training experimentally. While fairness can be improved, it often comes at the expense of reduced model utility.
\begin{figure}[ht]
    \centering
    \includegraphics[width=0.75\linewidth]{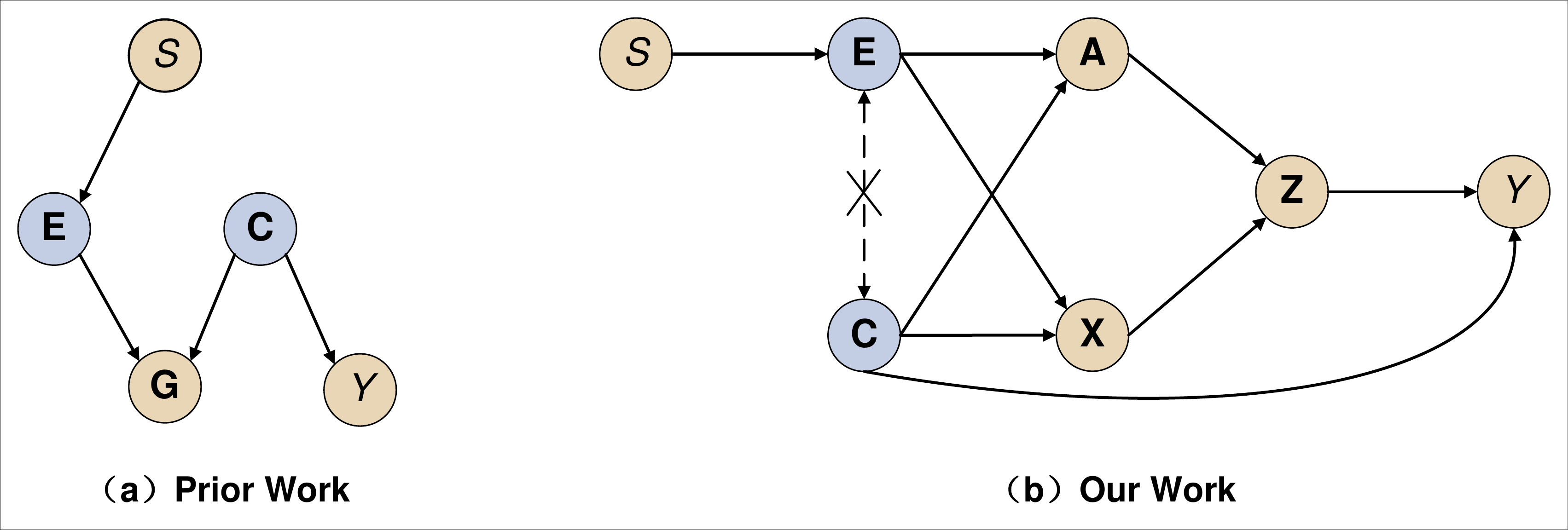} 
    \caption{\textbf{Two causal graphs are used to illustrate the comparison between prior work (a) and our proposed method (b). The causal graphs illustrate the predictions made by GNNs. Observed variables are depicted in brown, while unobserved variables are shown in blue. Solid edges represent direct causal relationships, whereas a dashed double arrow with a cross between \bm{\mathrm{E}} and \bm{\mathrm{C}} denotes statistical independence.}}
    \label{SCM}
\end{figure}

We first illustrate the causal relationships among seven distinct variables: observed sensitive attribute $S$, unobserved environmental feature \bm{\mathrm{E}}, unobserved content feature \bm{\mathrm{C}}, observed graph structure \bm{\mathrm{A}}, observed node feature \bm{\mathrm{X}}, node embedding \bm{\mathrm{Z}}, and true label $Y$. Specifically, $S$ $\rightarrow$ \bm{\mathrm{E}} indicates that the latent environmental feature \bm{\mathrm{E}} is decided by the sensitive attribute $S$. For instance, if $S$ represents gender, different genders may have distinct physical traits represented by \bm{\mathrm{E}}. The relationship \bm{\mathrm{C}} $\tikz[baseline] \draw[dashed, <->, >=to] (0,0.1) -- (0.35,0.1);$ \bm{\mathrm{E}} with a cross indicates that any spurious correlation between \bm{\mathrm{C}} and \bm{\mathrm{E}}  should be disentangled. Ideally,  \bm{\mathrm{C}} should represent content features that are independent of  \bm{\mathrm{E}}, ensuring that  \bm{\mathrm{C}} captures only information unrelated to the environmental features represented by  \bm{\mathrm{E}}. The relationships \bm{\mathrm{C}} $\rightarrow$ \bm{\mathrm{X}} $\leftarrow$ \bm{\mathrm{E}} and \bm{\mathrm{C}} $\rightarrow$ \bm{\mathrm{A}} $\leftarrow$ \bm{\mathrm{E}} demonstrate that the environmental feature \bm{\mathrm{E}} and content feature \bm{\mathrm{C}} jointly shape the observed contextual subgraph's node feature \bm{\mathrm{X}} and graph structure \bm{\mathrm{A}}. Furthermore, \bm{\mathrm{X}} $\rightarrow$ \bm{\mathrm{Z}} $\leftarrow$ \bm{\mathrm{A}} and \bm{\mathrm{Z}} $\rightarrow$ $Y$ denote that GNNs generate embeddings \bm{\mathrm{Z}} and predictions $Y$ grounded in the observed contextual subgraph.

Compared with conventional methods, our approach introduces key advancements at the causal modeling level \citep{cheng2024data, yao2021instance}. Existing fair graph representation learning techniques often fail to distinguish the different paths through which sensitive attributes propagate via node features and graph structures. This results in a mixture of bias sources and makes it difficult to achieve targeted intervention. As shown in Fig. \ref{SCM}(a), traditional methods typically depict the sensitive attribute $S$ influencing the subgraph \bm{\mathrm{G}} through a latent variable \bm{\mathrm{E}}, which in turn affects the prediction $Y$, however, these causal paths are neither explicitly modeled or systematically blocked. In contrast, as shown in Fig. \ref{SCM}(b),  our method leverages causal graphs to disentangle the influence of sensitive attributes into two distinct paths: the feature path ($S$ $\rightarrow$ \bm{\mathrm{E}} $\rightarrow$ \bm{\mathrm{X}} $\rightarrow$ \bm{\mathrm{Z}} $\rightarrow$ $Y$) and the structure path ($S$ $\rightarrow$ \bm{\mathrm{E}} $\rightarrow$ \bm{\mathrm{A}} $\rightarrow$ \bm{\mathrm{Z}} $\rightarrow$ $Y$). We also introduce intermediate variables, such as  content variable \bm{\mathrm{C}} and node embedding \bm{\mathrm{Z}}, to more accurately model representation learning and prediction mechanisms. Overall, our FairDTD approach provides a novel and principled perspective for advancing fairness in GNNs.

{\raggedright\textbf{Theoretical Analysis.}}\hspace{1em}As illustrated in Fig. \ref{SCM}, our constructed causal graph reveals two critical causal paths through which the sensitive attribute $S$ can influence the prediction $Y$: $S$ $\rightarrow$ \bm{\mathrm{E}} $\rightarrow$ \bm{\mathrm{X}} $\rightarrow$ \bm{\mathrm{Z}} $\rightarrow$ $Y$ and $S$ $\rightarrow$ \bm{\mathrm{E}} $\rightarrow$ \bm{\mathrm{A}} $\rightarrow$ \bm{\mathrm{Z}} $\rightarrow$ $Y$. These two paths constitute causal channels through which sensitive information may introduce bias into model predictions. To achieve this, it is essential to identify and block the critical transmission points along these paths that carry sensitive information. To formalize this, we introduce the following definition:

\renewcommand{\thedefinition}{\arabic{definition} (Path-Specific Blocking \citep{zhang2024path})}

\begin{definition}\label{dingyi1}
Given a causal path p, if we intervene on intermediate variables along this path (such as node features \bm{\mathrm{X}} or graph structure \bm{\mathrm{A}}) in a way that prevents the sensitive attribute from influencing the prediction outcome $Y$ through $p$, then the path is said to be specifically blocked.

\end{definition}

Building on this definition, we propose a path-decoupled strategy, which trains two models that each rely exclusively on a single type of pathway variable (either node features or graph structure) to achieve path-specific blocking and thereby enhance model fairness.

\begin{figure}[htbp] 
    \centering
    \begin{minipage}{0.49\linewidth}
        \centering
        \includegraphics[width=1\linewidth]{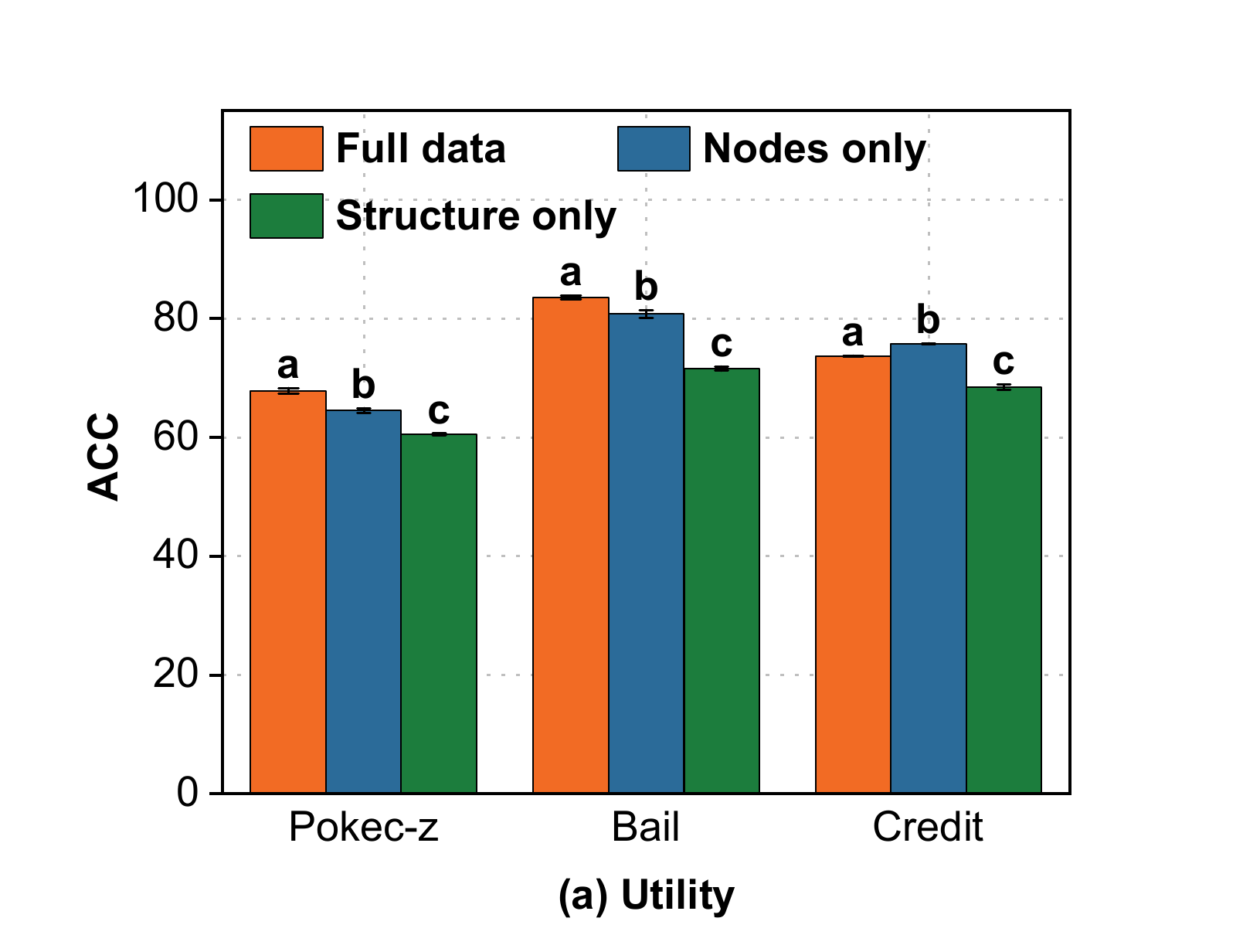} 
        
    \end{minipage}
    \hfill
    \begin{minipage}{0.49\linewidth}
        \centering
        \includegraphics[width=1\linewidth]{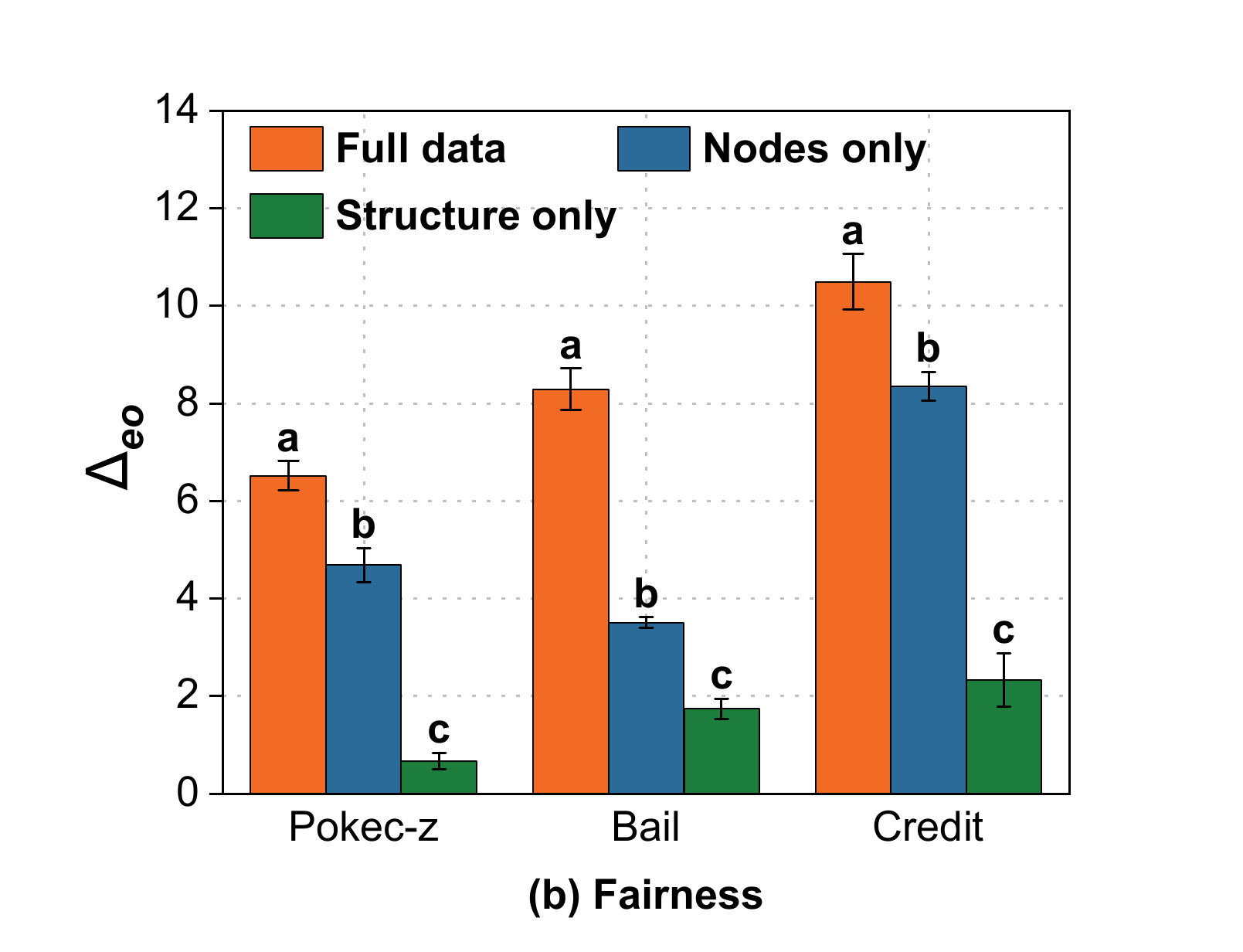} 
        
    \end{minipage}
    \caption{\textbf{The performance of model utility and fairness of the three training strategies, full data, only nodes, and only topology, on the Pokec-z, Bail, and Credit datasets.}}
    \label{partial}
\end{figure}

{\raggedright\textbf{Partial data training.}}\hspace{1em}Inspired by the above theoretical analysis, we further design experiments to validate our strategy—that training on partial data can effectively reduce bias in the representation caused by the sensitive attribute. Specifically, we analyzed the fairness performance of GNNs under different training strategies, as shown in Fig.~\ref{partial}.  Using node classification as the downstream task, the specific training strategies are as follows: (1) A two-layer GCN trained on complete data; (2) A two-layer MLP trained on only node features; (3) A two-layer GCN trained solely on the graph structure (with an all-one attribute matrix). The latter two strategies are collectively designated as partial data training.

The results in Fig.~\ref{partial} show that compared with complete data training, partial data training achieves improved fairness performance but at the expense of reduced model utility. A possible explanation for these results is that training on partial data reduces the source of bias, thereby improving fairness. However, compared to complete data training, the model lacks some key classification information, leading to diminished utility.

In summary, we identify the challenges confronted by partial data training: how does the model achieve the trade-off between fairness and model utility when training with partial data?

\section{Methodology}
To tackle the challenges related to training on partial data, we present a framework called FairDTD for learning fair GNNs through dual-teacher distillation of features and structure. We present a thorough description of FairDTD, with the framework presented in Fig.~\ref{FairDTD}. In dual-teacher distillation, we employ two fairness teachers of feature and structure to guide student learning. In graph-level distillation, we introduce intermediate layer representation alignment as additional supervision to guide the training of students. In learning node-specific temperatures, we adaptively soften the teacher's predictions based on the confidence levels of their outputs, facilitating effective knowledge transfer.

In the following sections, we present a thorough explanation of each component that constitutes the FairDTD framework. Finally, we present the training algorithm to enhance understanding of the FairDTD process.

\begin{figure*}[htbp]
    \centering
    \includegraphics[width=\textwidth]{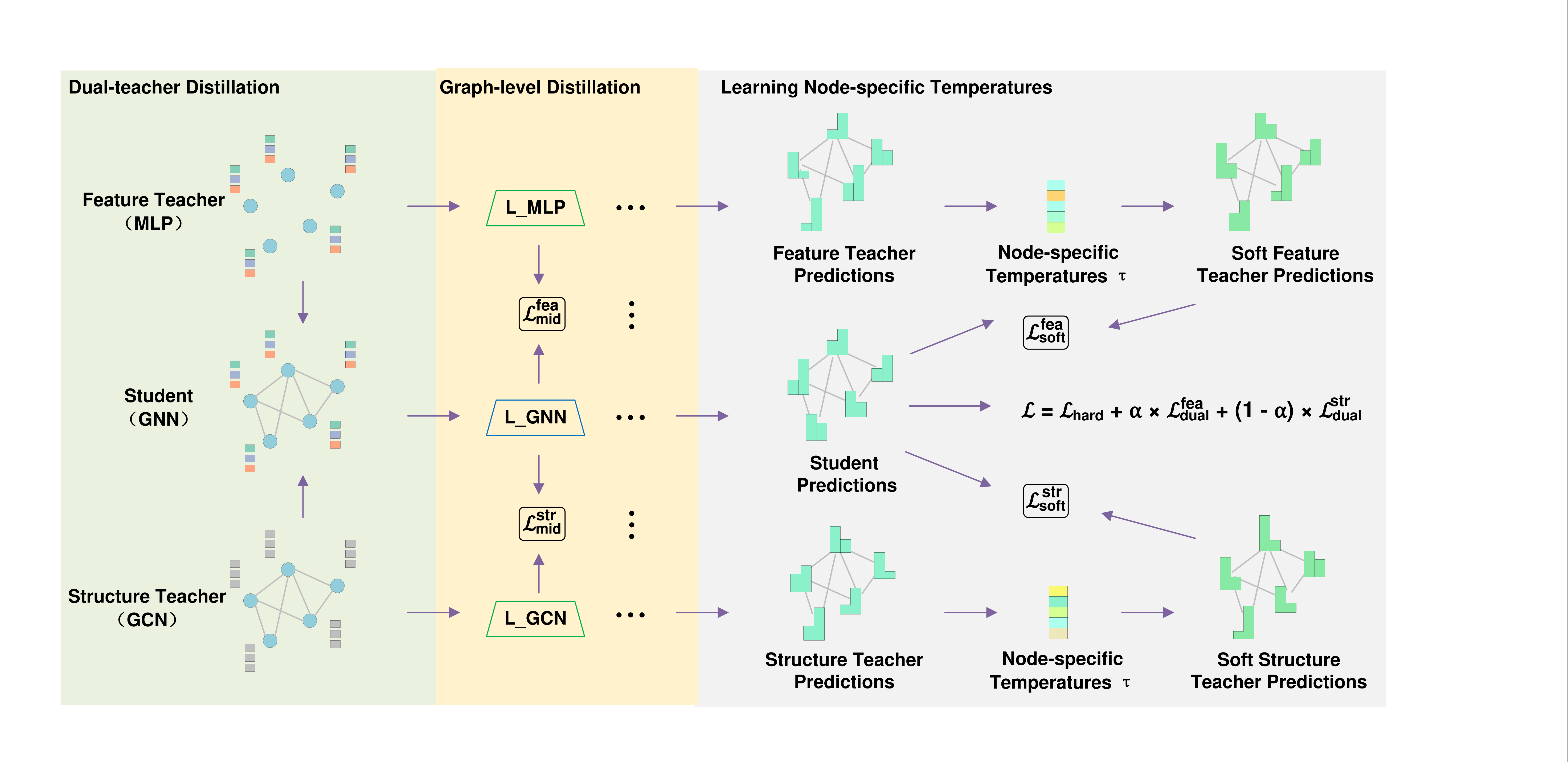} 
    \caption{\textbf{The proposed FairDTD framework consists of three components: (a) dual-teacher distillation, (b) graph-level distillation, and (c) learning node-specific temperatures. The GNN student learns fair and informative knowledge under the guidance of the feature teacher and the structure teacher.}}
    \label{FairDTD}
\end{figure*}

\subsection{Dual-teacher distillation based on feature and structure}
In this subsection, we describe the key components of dual-teacher distillation along with its associated optimization objective. Our approach is motivated by empirical analyses of partial data training. The primary challenge to address is avoiding the loss of utility caused by the lack of information inherent in partial data training. 

To tackle this issue, we employ two fairness-oriented teacher models, $f_{fea}$ and $f_{str}$, to provide targeted guidance to the student model. The fairness expert $f_{fea}$ takes only the node feature \bm{\mathrm{X}} as the input and obtains the output logits $\bm{\mathrm{Z}}_{fea}$ using a two-layer MLP, defined as follows:
\begin{equation}
\bm{\mathrm{Z}}_{fea}=f_{fea}(\bm{\mathrm{X}}),
\end{equation}
where $f_{fea}$ is a two-layer MLP with the Relu activation function, and its training objective is as follows:
\begin{align}
\min _{\theta_{fea}} \mathcal{L}_{fea} = -\mathbb{E}_{v_i \sim \mathcal{V}} \Big( y_i \log \left( \sigma(z_i^{fea}) \right) + (1 - y_i) \log \left( 1 - \sigma(z_i^{fea}) \right) \Big).
\label{eq6}
\end{align}

The fair expert $f_{str}$ takes only the graph structure \bm{\mathrm{A}} (all-one node attribute matrix) as the input and then obtains the output logits $\bm{\mathrm{Z}}_{str}$ through a two-layer GCN, defined as follows:
\begin{equation}
\bm{\mathrm{Z}}_{str}=f_{str}(\bm{\mathrm{A}}),
\end{equation}
where $f_{str}$ is a two-layer GCN, and its training objective is defined as follows:
\begin{align}
\min _{\theta_{str}} \mathcal{L}_{str} = -\mathbb{E}_{v_i \sim \mathcal{V}} \Big( y_i \log \left( \sigma(z_i^{str}) \right) + (1 - y_i) \log \left( 1 - \sigma(z_i^{str}) \right) \Big).
\label{eq8}
\end{align}

This framework ensures that the feature-based teacher model $f_{fea}$  focuses on fairness-driven learning from node-level information, effectively complementing the structural guidance provided by  $f_{str}$. 

As discussed in Section \ref{analysis}, $f_{fea}$ and $f_{str}$ help mitigate the source of bias, ensuring that $\bm{\mathrm{Z}}_{fea}$ and $\bm{\mathrm{Z}}_{str}$ contain fairer node prediction information. However, since $f_{fea}$ and $f_{str}$ are trained on partial data, the accuracy of these predictions may be compromised. 

To address this, we employ a dual-teacher model, aiming to provide comprehensive supervisory information.
We design a dual distillation loss to maximize the guidance provided to the student model, balancing fairness and utility. The student model is motivated to imitate the predictions of the fairness-oriented teacher models, thereby facilitating the transfer of knowledge. The student model $f_{stu}$ takes complete data as the input and  obtains the output logits $\bm{\mathrm{Z}}_{stu}$, defined as follows:
\begin{equation}
\bm{\mathrm{Z}}_{stu}=f_{stu}(\bm{\mathrm{X}}, \bm{\mathrm{A}}),
\end{equation}
where $f_{stu}$ is a two-layer GNN, and its training objective is defined as follows:
\begin{align}
\min _{\theta_{stu}} \mathcal{L}_{hard} = -\mathbb{E}_{v_i \sim \mathcal{V}} \Big( y_i \log \left( \sigma(z_i^{stu}) \right) + (1 - y_i) \log \left( 1 - \sigma(z_i^{stu}) \right) \Big).
\label{eq10}
\end{align}

Finally, leveraging $\bm{\mathrm{Z}}_{fea}$, $\bm{\mathrm{Z}}_{str}$ and $\bm{\mathrm{Z}}_{stu}$, the dual-teacher distillation loss is defined as follows:
\begin{equation}
\mathcal{L}_{soft}^{fea}=\mathcal{H}_{K L}\left(\operatorname{softmax}(\frac{\bm{\mathrm{Z}}_{stu}}{\tau}) \| \operatorname{softmax}(\frac{\bm{\mathrm{Z}}_{fea}}{\tau})\right),
\label{eq11}
\end{equation}
\vspace{-0.3cm}
\begin{equation}
\mathcal{L}_{soft}^{str}=\mathcal{H}_{K L}\left(\operatorname{softmax}(\frac{\bm{\mathrm{Z}}_{stu}}{\tau}) \| \operatorname{softmax}(\frac{\bm{\mathrm{Z}}_{str}}{\tau})\right),
\label{eq12}
\end{equation}
where $\tau$ denotes a temperature parameter that regulates the smoothness of the output distributions.

\subsection{Graph-level distillation}
In this subsection, we introduce graph-level distillation to further enhance the teacher-student model. While the dual-teacher distillation generates soft labels through the output of the last layer to provide additional supervisory signals for the student, it does not account for intermediate supervision within the teacher model. This limitation may hinder the model’s ability to learn fair representations effectively, as deep neural networks excel at capturing feature representations at various levels in multi-layer architectures \citep{bengio2013representation}.

To address this, we adopt intermediate representation distillation as an additional supervisory signal to guide the student model’s training. This method also indirectly supplements the missing information caused by partial data training. In FairDTD, we use the intermediate graph-level representations $\bm{\mathrm{R}}_{fea}$ and $\bm{\mathrm{R}}_{str}$ from the dual-teacher model as supervisory signals. These representations encapsulate more latent information, enabling the student model to learn effectively. We redefine Eqs.~(5), (7), and (9) to include intermediate representations as follows:
\begin{equation}
\bm{\mathrm{Z}}_{fea}, \bm{\mathrm{R}}_{fea} =f_{fea}(\bm{\mathrm{X}}), \hspace{1cm} \bm{\mathrm{Z}}_{str}, \bm{\mathrm{R}}_{str}=f_{str}(\bm{\mathrm{A}}), \hspace{1cm} \bm{\mathrm{Z}}_{stu}, \bm{\mathrm{R}}_{stu}=f_{stu}(\bm{\mathrm{X}}, \bm{\mathrm{A}}).
\end{equation}

In particular, our goal is to maximize the consistency of the intermediate representations of the teacher and student models. The graph-level distillation loss is formulated as follows:
\begin{equation}
\mathcal{L}_{mid}^{fea}=\|\frac{\bm{\mathrm{R}}_{stu}}{\|\bm{\mathrm{R}}_{stu}\|_2}-\frac{\bm{\mathrm{R}}_{fea}}{\|\bm{\mathrm{R}}_{fea}\|_2}\|_2^2,
\label{eq16}
\end{equation}
\vspace{-0.3cm}
\begin{equation}
\mathcal{L}_{mid}^{str}=\|\frac{\bm{\mathrm{R}}_{stu}}{\|\bm{\mathrm{R}}_{stu}\|_2}-\frac{\bm{\mathrm{R}}_{str}}{\|\bm{\mathrm{R}}_{str}\|_2}\|_2^2,
\label{eq17}
\end{equation}
where $\| \cdot \|_2$ denotes the $L_2$ norm,  which measures the difference between the normalized intermediate representations.

\subsection{Learning node-specific temperatures}
In this subsection, we introduce the concept of learning node-specific temperatures to enhance knowledge transfer from teacher to student models. While dual-teacher distillation and graph-level distillation ensure that teacher models encode fair and information-rich representations, a critical remaining challenge is ensuring effective transfer of this knowledge to the student model, thereby improving performance.

Learning node-specific temperatures addresses the limitations inherent in the temperature mechanism used in traditional distillation methods. In conventional knowledge distillation, a predefined temperature parameter is used to generate soft targets. The teacher model's informative dark knowledge is embedded in the soft targets, which infer the probability that a node belongs to a given class. The softmax function with a fixed temperature is defined as follows:
\vspace{-0.3cm}
\begin{equation}
p\left(\bm{\mathrm{Z}}, \tau\right)=\frac{\exp \left(\bm{\mathrm{Z}} / \tau\right)}{\sum \exp \left(\bm{\mathrm{Z}} / \tau\right)},
\end{equation}
where the temperature parameter $\tau$ is used to control the softness of each prediction. The distribution of predictions becomes smoother as the temperature rises, and sharper as the temperature falls. Thus, the temperature is responsible for regulating the balance between the true label knowledge and the dark knowledge.

However, using a fixed temperature treats the softening of all teacher logits equally, ignoring variations in node representations. This uniform approach fails to provide precise guidance for individual nodes, which can hinder knowledge transfer. Moreover, existing research has shown that temperature scaling can improve model performance \citep{zhang2020self}.

To confront these issues, we propose an entropy-based approach to learn node-specific temperatures. This approach assigns a unique temperature $\tilde{\tau}_i$ to each node, controlling the degree of softening individually. The softmax function with node-specific temperatures is defined as:
\begin{equation}
p\left(z_i, \tilde{\tau}_i\right)=\frac{\exp \left(z_i / \tilde{\tau}_i\right)}{\sum \exp \left(z_i / \tilde{\tau}_i\right)},
\end{equation}
where the temperature parameter $\tilde{\tau}_i$ is the specific temperature of the $i$-th node. Obviously, $\tilde{\tau}_i$ determines the softening direction of each node. The temperature of the node is related to the teacher's confidence for each respective node \citep{guo2023boosting}. 

Therefore, we can judge the teacher's confidence for each node through the entropy of the teacher's logits. The lower the entropy value, the higher the teacher's confidence for this node. Specifically, we calculate the node-specific temperature through the probability distribution of the teacher and the confidence of the teacher, defined as follows:
\begin{equation}
\operatorname{Entropy}\left(\tilde{z}_i\right)=-\sum_{c=1}^C z_i^c \log \left(z_i^c\right),
\label{eq20}
\end{equation}
\vspace{-0.3cm}
\begin{equation}
\tilde{\tau}_i=\operatorname{softmax}\left(\operatorname{MLP}\left(\operatorname{Concat}\left(\tilde{z}_i, \operatorname{Entropy}\left(\tilde{z}_i\right)\right)\right)\right),
\label{eq21}
\end{equation}
where $z_i^c$ is the probability that the node $i$ belongs to class $c$ and $\operatorname{Concat}$ represents the concatenation function. We use $\tilde{\tau}_i = \tilde{\tau}_i \times (\tau_{max} - \tau_{min}) + \tau_{min}$ to constrain the temperature within a predetermined range $[\tau_{min}, \tau_{max}]$.

In summary, our method assigns lower temperatures to predictions with higher entropy (lower confidence) and higher temperatures to predictions with lower entropy (higher confidence). This adjustment ensures sharper predictions are softened more, while overly smooth predictions are sharpened, facilitating more effective knowledge transfer from teacher models to the student model.

\subsection{Fair GNN student training}
In this subsection, we introduce the final training objective for the fair GNN student. FairDTD trains a fair student model $f_{stu}$ using the trained $f_{fea}$ and $f_{str}$ as dual teacher models. The primary objective is for the student model to imitate the outputs of the fair dual-teacher models. To achieve this, we combine the dual distillation loss and the graph-level loss into the final distillation loss between the student and teacher models, defined as:
\begin{equation}
\mathcal{L}_{dual}^{fea}=\mathcal{L}_{soft}^{fea} + \mathcal{L}_{mid}^{fea},
\label{eq22}
\end{equation}
\vspace{-0.3cm}
\begin{equation}
\mathcal{L}_{dual}^{str}=\mathcal{L}_{soft}^{str} + \mathcal{L}_{mid}^{str}.
\label{eq23}
\end{equation}

By incorporating the student model’s own training loss, the final complete training loss is expressed as:
\begin{equation}
\mathcal{L}_{final}^{stu}=\mathcal{L}_{hard} + \alpha \mathcal{L}_{dual}^{fea} + (1-\alpha) \mathcal{L}_{dual}^{str},
\label{eq24}
\end{equation}

where $\alpha$ is a hyperparameter that balances the feature teacher and the structural teacher.

To facilitate a clearer understanding of FairDTD, we provide its complete training process in Algorithm \ref{algorithm}.

\subsection{Theorem and proof for FairDTD}

In this subsection, we present Theorem \ref{theorem1} to formally derive the feasibility of FairDTD’s path-specific blocking approach under the dual fairness teacher distillation framework, and prove that it achieves fairness at the causal level, as detailed below:

\renewcommand{\proofname}{\textbf{Proof}}

\begin{theorem}
\label{theorem1}
If multiple undesirable causal paths exist between the sensitive attribute $S$ and the prediction output $Y$, and each path relies on a different intermediate variable (e.g., node features \bm{\mathrm{X}} or graph structure \bm{\mathrm{A}}), then training models that each use only a single intermediate variable allows for path-specific blocking. This strategy weakens the overall impact of $S$ on the learned representation \bm{\mathrm{Z}}, thereby improving the fairness.
\end{theorem}

\begin{proof}
Consider the following two causal paths: $p_1$: $S$ $\rightarrow$ \bm{\mathrm{E}} $\rightarrow$ \bm{\mathrm{X}} $\rightarrow$ \bm{\mathrm{Z}} $\rightarrow$ $Y$ and $p_2$: $S$ $\rightarrow$ \bm{\mathrm{E}} $\rightarrow$ \bm{\mathrm{A}} $\rightarrow$ \bm{\mathrm{Z}} $\rightarrow$ $Y$. We construct two models that each rely exclusively on a single intermediate variable: (1) A feature-based model, which learns the representation $\bm{\mathrm{Z}}_{fea}$ solely from node features \bm{\mathrm{X}}. Since it does not depend on the graph structure \bm{\mathrm{A}}, path $p_2$ is explicitly blocked. (2) A structure-based model, which learns the representation $\bm{\mathrm{Z}}_{str}$ solely from the graph structure \bm{\mathrm{A}}. As it excludes node features \bm{\mathrm{X}}, path $p_1$ is effectively blocked. If the final representation \bm{\mathrm{Z}} is learned by combining both models, then by the sub-additivity property of mutual information, we have:
\begin{equation}
I(\bm{\mathrm{Z}}; S) \leq I(\bm{\mathrm{Z}}_{fea}; S) + I(\bm{\mathrm{Z}}_{str}; S),
\end{equation}
where $I(\bm{\mathrm{Z}}; S)$ denotes the mutual information between the learned representation $\bm{\mathrm{Z}}$ and the sensitive attribute $S$. A smaller value of $I(\bm{\mathrm{Z}}; S)$ indicates that less sensitive information is encoded in \bm{\mathrm{Z}}. This demonstrates that blocking the transmission of the sensitive attribute $S$ along each specific causal path can effectively reduce the overall dependency between $S$ and the final representation $\bm{\mathrm{Z}}$, thereby enhancing fairness at the causal level.
\end{proof}

\begin{algorithm}[t]
    \SetKwInOut{Input}{Input}  
    \SetKwInOut{Output}{Output}  
    \caption{Training Algorithm of FairDTD} \label{algorithm}
    \Input{$\mathcal{G} = (\mathcal{V}, \mathcal{E}, \mathrm{X})$, a two-layer MLP feature teacher $f_{fea}$, a two-layer GCN structure teacher $f_{str}$, a two-layer GNN student $f_{stu}$, $\alpha$, $\tau_{max}$, $\tau_{min}$.}
    \Output{The GNN student model trained with parameter $\theta_{stu}$.}

    // \textbf{Feature teacher training}
    
    \While{not converge}{
        \For{1 $\rightarrow epoch$ to $epoch_{fea}$}{
            $\bm{\mathrm{Z}}_{fea}, \bm{\mathrm{R}}_{fea} \leftarrow f_{fea}(\bm{\mathrm{X}})$;
            
            Calculate $\mathcal{L}_{fea}$ according to Eq. \eqref{eq6};

            Back-propagation to update parameters $\theta_{fea}$; 
        }
    $\mathbf{end while}$
    }
        
    // \textbf{Structure teacher training}

    \While{not converge}{
        \For{1 $\rightarrow epoch$ to $epoch_{str}$}{
            $\bm{\mathrm{Z}}_{str}, \bm{\mathrm{R}}_{str} \leftarrow f_{str}(\bm{\mathrm{A}})$;
            
            Calculate $\mathcal{L}_{str}$ according to Eq. \eqref{eq8};

            Back-propagation to update parameters $\theta_{str}$; 
        }
    $\mathbf{end while}$
    }

    // \textbf{Fair GNN student training}

    \While{not converge}{
        \For{1 $\rightarrow epoch$ to $epoch_{stu}$}{
            $\bm{\mathrm{Z}}_{stu}, \bm{\mathrm{R}}_{stu} \leftarrow f_{stu}(\bm{\mathrm{X}}, \bm{\mathrm{A}})$;

             Calculate $\tilde{\tau}_i^{fea}$ and $\tilde{\tau}_i^{str}$ based on $\bm{\mathrm{Z}}_{fea}$ and $\bm{\mathrm{Z}}_{str}$ according to Eq. \eqref{eq20} and Eq. \eqref{eq21};

             Calculate $\mathcal{L}_{soft}^{fea}$ and $\mathcal{L}_{soft}^{str}$ based on $\tilde{\tau}_i^{fea}$ and $\tilde{\tau}_i^{str}$ according to Eq. \eqref{eq11} and Eq. \eqref{eq12};

             Calculate $\mathcal{L}_{mid}^{fea}$ and $\mathcal{L}_{mid}^{str}$ based on $\bm{\mathrm{R}}_{fea}$, $\bm{\mathrm{R}}_{str}$ and $\bm{\mathrm{R}}_{stu}$ according to Eq. \eqref{eq16} and Eq. \eqref{eq17};

             Calculate $\mathcal{L}_{dual}^{fea}$ based on $\mathcal{L}_{soft}^{fea}$ and $\mathcal{L}_{mid}^{fea}$ according to Eq. \eqref{eq22};

             Calculate $\mathcal{L}_{dual}^{str}$ based on $\mathcal{L}_{soft}^{str}$ and $\mathcal{L}_{mid}^{str}$ according to Eq. \eqref{eq23};
            
            Calculate $\mathcal{L}_{hard}$ according to Eq. \eqref{eq10};

            Calculate $\mathcal{L}_{final}^{stu}$ according to Eq. \eqref{eq24};

            Back-propagation to update parameters $\theta_{stu}$; 
        }
    $\mathbf{end while}$
    }

    \textbf{return} $\theta_{stu}$.
\end{algorithm}

\section{Experiments}
In this section, we present comprehensive experiments to assess the efficacy of the suggested approach. Specifically, we perform experiments on multiple fairness datasets to investigate and address the following questions:

$\mathbf{RQ1}$: Compared with the baseline methods, can FairDTD maintain the utility performance while improving fairness? $\mathbf{RQ2}$: What is the impact of each component within the FairDTD framework on the overall model performance? $\mathbf{RQ3}$: Does FairDTD generate node representations that effectively distinguish between different groups compared to vanilla? $\mathbf{RQ4}$: How does the time cost of FairDTD compare to the baseline methods? $\mathbf{RQ5}$: How do the relevant hyperparameters affect FairDTD?

\subsection{Experimental settings}
\subsubsection{Datasets}
We use three commonly used real-world datasets for our experiments: Pokec-z, Pokec-n, and Credit. The statistical details of each dataset are summarized in Table \ref{dataset}. A brief overview of the datasets is as follows:
\vspace{-1.2ex}
\begin{itemize}[leftmargin=*,labelindent=2em,itemsep=-0.1em]
    \item \textbf{Pokec-z/n} \citep{dai2021say, takac2012data} is sourced from Pokec, a widely used social networking platform in Slovakia. These datasets reflect social network data collected from two distinct provinces. Where nodes represent individual users, while edges denote the follower relationships between them. The ``region'' is treated as the sensitive attribute, and the users' working fields are binarized to serve as the label for prediction.
    \item \textbf{Credit} \citep{yeh2009comparisons} is a dataset containing information related to credit card users. Where nodes stand in for credit card users, and edges are based on how similar customers' payment habits are. The aim is to estimate if a user would be able to make timely repayments of their credit card debt in the upcoming month, with ``age'' being regarded as the sensitive attribute.
\end{itemize}

\begin{table}[htbp]
    \centering
    \caption{\textbf{An overview of the statistics for the three real-world datasets.}}
    \label{dataset}
    \begin{tabular}{llll}
        \toprule 
        \textbf{Dataset} & \textbf{Pokec-z} & \textbf{Pokec-n}  & \textbf{Credit} \\
        \midrule 
        \#Nodes & 7,659 & 6,185  & 30,000 \\
        
        \#Edges & 29,476 & 21,844  & 2,843,716 \\
        
        \#Features & 59 & 59  & 13 \\
        
        Sens. & Region & Region  & Age \\
        
        Label & Working field & Working field  & Future default \\
        
        \bottomrule 
    \end{tabular}
\end{table}

\subsubsection{Baselines}
We conduct a comparative analysis of FairDTD against four state-of-the-art fairness methodologies: NIFTY, FairVGNN, CAF, FairGKD and FairGP. A brief overview of these methodologies is as follows:
\vspace{-1.2ex}
\begin{itemize}[leftmargin=*,labelindent=2em,itemsep=-0.1em]
    \item \textbf{NIFTY} \citep{agarwal2021towards} enforces counterfactual fairness constraints by maximizing the similarity between the counterfactual graph generated by flipping the sensitive attribute and the original graph, but it overlooks latent biases within the graph structure and struggles to produce realistic and reliable counterfactuals.
    \item \textbf{FairVGNN} \citep{wang2022improving} mitigates the leakage of sensitive attributes by masking the sensitive-related channels and clamping the encoder weights but lacks explicit modeling of the sensitive attribute propagation mechanism.
    \item \textbf{CAF} \citep{guo2023towards} aims to identify counterfactuals from the training data and uses them as supervisory signals for learning fair node representations through disentangled representation learning. However, this approach depends on the quality of the counterfactual samples and has limited scalability to large-scale graphs.
    \item \textbf{FairGKD} \citep{zhu2024devil} constructs a synthetic teacher with contrastive learning and a fixed temperature parameter for distillation. While informative, this setup may constrain effective knowledge transfer.
    \item \textbf{FairGP} \citep{luo2025fairgp} proposes a partition-based fairness-aware graph transformer model that mitigates the influence of high-order sensitive nodes and optimizes the global attention mechanism, significantly enhancing fairness in graph representation learning while improving computational efficiency. While computationally efficient, it focuses on structural interventions and lacks a causal framework.
\end{itemize}

\subsubsection{Evaluation metrics}
We use the accuracy of the model to assess its utility. For fairness performance, we adopt two widely recognized fairness measures: Statistical Parity (SP) \citep{dwork2012fairness} and Equal Opportunity (EO) \citep{hardt2016equality}. To achieve statistical parity, we use $\Delta_{sp}$ as the evaluation metric. Specifically, it is defined as $\Delta_{sp}=|P(\hat{y}=1 \mid s=0)-P(\hat{y}=1 \mid s=1)|$, which assesses the ratio of instances from different groups of sensitive attributes that are categorized as positive or negative. Similarly, for equal opportunity, we use $\Delta_{eo}$ as the evaluation metric. It is defined as $\Delta_{eo}=|P(\hat{y}=1 \mid y=1, s=0)-P(\hat{y}=1 \mid y=1, s=1)|$, which evaluates if the model generates consistent predictions for people with similarly non-sensitive attributes. Smaller values of $\Delta_{sp}$ and $\Delta_{eo}$ indicate a fairer model.

\subsubsection{Implementation details}
To assess the generalizability of FairDTD across various architectures, we employ GCN and GIN as the encoder backbones. For consistency across all datasets, the hidden dimension for each GNN backbone is uniformly set to 64. For baseline methods, we use results obtained from their respective open-source code repositories, adhering to the hyperparameter settings provided by the authors. The experiments are repeated five times, and the average outcomes are documented.

In FairDTD, we utilize a two-layer MLP for the feature teacher, while both the structure teacher and the student model are constructed using a two-layer GCN or GIN. For the GCN and GIN backbones, we use the Adam optimizer during training, with learning rates (lr) of \{0.01, 0.0001\}. The number of training epochs is set to 700. In addition, we adjust the balance parameter $\alpha$ mentioned in the method in \{0.1, 0.2, ..., 0.9\}. A single NVIDIA A800 GPU with 80GB memory is used for all experiments. The PyTorch framework is used to implement the models.

\begin{table*}[htbp]
    \centering
    \caption{\textbf{A comparative analysis of our proposed method, FairDTD, against baseline approaches is presented, focusing on accuracy and fairness metrics. The best performance for each backbone GNN is highlighted in bold.}}
    \label{baseline}
    \resizebox{\textwidth}{!}{
        \begin{tabular}{c|c|ccc|ccc|ccc}  
            \toprule
            \multirow{2}{*}{\textbf{Encoder}} & \multirow{2}{*}{\textbf{Method}} & & \textbf{Pokec-z} & & & \textbf{Pokec-n} & & & \textbf{Credit} & \\ 
            \cmidrule{3-11}  
            & & ACC ↑ & \textit{$\Delta_{sp}$} ↓ & \textit{$\Delta_{eo}$} ↓ & ACC ↑ & \textit{$\Delta_{sp}$} ↓ & \textit{$\Delta_{eo}$} ↓ & ACC ↑ & \textit{$\Delta_{sp}$} ↓ & \textit{$\Delta_{eo}$} ↓ \\ 
            \cmidrule{1-11} 
            \multirow{6}{*}{GCN} & GCN & 67.83±0.47 & 7.31±0.59 & 6.52±0.30 & 67.24±0.33 & 6.39±0.39 & 7.27±0.67 & 73.37±0.02 & 12.65±0.18 & 10.49±0.57 \\ 
            & NIFTY & 67.25±0.28 & 4.38±1.02 & 3.89±0.72 & 66.58±0.37 & 5.61±0.24 & 4.15±0.86 & 73.41±0.04 & 11.63±0.07 & 9.02±0.42  \\ 
            & FairVGNN & 68.39±3.01 & 3.63±1.75 & 4.74±1.94 & 68.11±0.35 & 4.13±1.65 & 6.14±3.04 & 78.51±0.39 & 4.39±3.75  & 2.70±2.72  \\ 
            & CAF & 67.90±0.82 & 2.70±0.14 & 3.67±0.12 & 67.25±1.05 & 2.12±1.50 & 2.10±1.70 & 73.90±1.52 & 5.83±1.78  & 5.60±1.42  \\ 
            & FairGKD & 69.18±0.54 & 4.05±1.09 & 3.06±1.03 & 67.79±0.35 & 0.74±0.48 & 2.07±1.03 & 79.63±0.39 & 3.42±1.61  & 2.78±1.12  \\
            & FairGP & 68.39±0.43 & 2.58±0.21 & 1.86±0.52 & 68.09±0.39 & 1.06±0.42 & 2.20±1.32 & 79.31±0.64 & 2.62±0.41 & \textbf{1.70±0.27}  \\
            \cmidrule{2-11} 
            & FairDTD & \textbf{69.71±0.15} & \textbf{1.93±1.03} & \textbf{1.76±0.82} & \textbf{69.49±0.53} & \textbf{0.49±0.38} & \textbf{1.98±0.55} & \textbf{80.83±0.30} & \textbf{2.41±0.26}  & 1.90±0.14  \\
            \cmidrule{1-11}  
            \multirow{6}{*}{GIN} & GIN & 66.82±0.53 & 5.14±0.52 & 4.56±0.50 & 66.54±0.93 & 5.92±1.48 & 4.56±1.68 & 73.65±0.67 & 13.24±3.08 & 10.45±3.70 \\ 
            & NIFTY & 65.57±1.34 & 2.70±1.28 & 3.23±1.92 & 66.37±1.51 & 3.84±1.05 & 3.24±1.60 & 74.43±0.42 & 6.01±4.19  & 5.41±3.03  \\ 
            & FairVGNN & \textbf{68.24±0.95} & 2.23±1.44 & 3.88±1.19 & 67.43±1.31 & 4.04±2.56 & 6.92±3.87 & 77.50±0.89 & 2.62±0.97  & 1.71±0.78  \\ 
            & CAF & 67.71±0.75 & 2.64±0.78 & 2.30±0.31 & 67.57±1.87 & 4.10±0.64 & 2.87±1.32 & 74.57±1.06 & 4.26±1.29  & 6.32±1.81  \\ 
            & FairGKD & 67.58±2.05 & 1.83±0.93 & 2.23±1.02 & 68.36±0.53 & 1.37±0.84 & 2.99±1.53 & 79.18±0.65 & 3.08±3.53  & 2.57±3.10  \\
            & FairGP & 67.57±0.60 & 1.67±0.37 & 1.29±0.18 & \textbf{68.81±0.82} & 1.85±0.54 & 1.69±0.28 & 78.21±0.33 & 1.30±0.28 & 1.58±0.67  \\
            \cmidrule{2-11} 
            & FairDTD & 67.73±0.77 & \textbf{1.01±0.70} & \textbf{0.94±0.85} & 67.74±0.71 & \textbf{1.16±0.98} & \textbf{1.27±0.54} & \textbf{80.29±0.27} & \textbf{1.06±0.52}  & \textbf{1.18±0.53} \\
            \bottomrule
        \end{tabular}
        
    }
\end{table*}

\subsection{Performance comparison}
To answer \textbf{RQ1}, we use GCN or GIN as the backbone models and compare FairDTD against four baseline approaches for the node classification task. Table \ref{baseline} displays the utility and fairness of the baseline approach as well as FairDTD. From this table, we make the following key observations:
 
\begin{enumerate}
\item Compared with the baseline methods, FairDTD achieves the best fairness on the three datasets, which proves the effectiveness of FairDTD in learning fair GNN methods by reducing the sources of bias. 

\item FairDTD maintains a comparable or even better classification performance, which proves that FairDTD can solve the problem of the decline in model utility due to the training with partial data and also indicates that FairDTD effectively achieves an improved balance between fairness and model utility. 

\item Compared with NIFTY, FairDTD avoids graph-level generation by modeling sensitive attribute propagation through causal paths, enabling more efficient and scalable fairness interventions. Compared with FairVGNN, FairDTD addresses this gap by using a causal graph to trace both feature-based and structure-based pathways of sensitive influence, enabling targeted, theory-driven blocking strategies. Compared with CAF, FairDTD uses principled causal interventions instead, enhancing generalization and robustness without depending on external counterfactual data. Compared with FairGKD's synthetic fair teacher model, FairDTD demonstrates superior performance across nearly all aspects. This highlights the effectiveness of our dual fair teacher model in providing richer knowledge to the student model during training. Additionally, unlike FairGKD's use of a fixed temperature, FairDTD learns node-specific temperatures to adaptively soften the teacher logits, thereby providing more accurate guidance for the student model and achieving a better balance between fairness and utility. Compared with FairGP, FairDTD explicitly models and blocks causal pathways of sensitive attributes via both node features and structure, supporting more targeted fairness control.

\item In most cases, FairDTD achieves optimal performance on both GCN and GIN backbones, demonstrating its strong generalization capability across different GNN architectures, which allows for flexible application in various scenarios. Additionally, we can observe that the performance of GIN in classification has declined compared to GCN. This might be due to the following reasons: Compared to GCN, which aggregates feature information through simple convolution operations, GIN introduces a more complex domain aggregation mechanism and MLP for representation learning for the graph isomorphism problem. This may cause excessive aggregation and instead blur the discrimination ability of nodes. 

 \item FairDTD achieves the best performance on the Credit dataset for both GCN and GIN. We speculate that this discrepancy may be attributed to the increased complexity of the training data in the Credit dataset, relative to the Pokec-z and Pokec-n datasets. As a more complicated model, GIN can fully utilize its stronger representational ability when processing complex data and thus performs well in this scenario.
\end{enumerate}

\begin{table*}[htbp]
    \centering
    \caption{\textbf{Analysis of model fairness and node classification bias among distinct FairDTD variants. The optimal results for each GNN backbone are highlighted in bold.}}
    \label{ablation}
    \resizebox{\textwidth}{!}{
        \begin{tabular}{c|l|ccc|ccc|ccc}  
            \toprule
            \multirow{2}{*}{\textbf{Encoder}} & \multirow{2}{*}{\textbf{Model Variants}}  & & \textbf{Pokec-z} & & & \textbf{Pokec-n} & & & \textbf{Credit} & \\ 
            \cmidrule{3-11}  
            & & ACC ↑ & \textit{$\Delta_{sp}$} ↓ & \textit{$\Delta_{eo}$} ↓ & ACC ↑ & \textit{$\Delta_{sp}$} ↓ & \textit{$\Delta_{eo}$} ↓ & ACC ↑ & \textit{$\Delta_{sp}$} ↓ & \textit{$\Delta_{eo}$} ↓ \\ 
            \cmidrule{1-11} 
            \multirow{6}{*}{GCN} & FairDTD & \textbf{69.71±0.15} & \textbf{1.93±1.03} & \textbf{1.76±0.82} & \textbf{69.49±0.53} & \textbf{0.49±0.38} & \textbf{1.98±0.55} & \textbf{80.83±0.30} & \textbf{2.41±0.26}  & \textbf{1.90±0.14}  \\
            & w/o FT & 66.88±0.94 & 2.69±0.70 & 2.41±0.23 & 67.68±0.42 & 2.15±0.45 & 3.18±0.52 & 77.08±0.29 & 3.86±0.48 & 2.60±0.23 \\ 
            & w/o ST & 69.35±0.97 & 3.43±1.67 & 3.36±1.56 & 68.88±1.58 & 3.71±0.29 & 4.41±0.51 & 78.76±0.30 & 4.84±0.15 & 3.46±0.35  \\ 
            & w/o GD & 68.83±1.20 & 5.82±0.25 & 5.33±1.58 & 68.71±0.68 & 2.94±0.22 & 3.24±2.74 & 78.60±1.14 & 5.60±1.50 & 3.06±0.62 \\ 
            & w/o NST & 68.74±0.60 & 2.64±1.11 & 3.16±2.30 & 68.52±0.56 & 2.01±0.84 & 4.08±1.60 & 79.11±0.78 & 4.40±0.63 & 2.54±0.79 \\ 
            \cmidrule{1-11}  
            \multirow{6}{*}{GIN} & FairDTD & 67.73±0.77 & \textbf{1.01±0.70} & 0.94±0.85 & \textbf{67.74±0.71} & \textbf{1.16±0.98} & \textbf{1.27±0.54} & \textbf{80.29±0.27} & \textbf{1.06±0.52}  & 1.18±0.53 \\
            & w/o FT & 65.60±1.36 & 3.22±1.63 & \textbf{0.44±0.26}   & 63.24±1.01 & 1.92±2.38 & 1.89±0.40 &   76.47±0.81 & 1.78±1.54 & \textbf{0.81±0.38} \\ 
            & w/o ST & \textbf{68.44±0.40} & 4.93±1.30 & 3.35±0.48   & 67.11±0.67 & 1.70±0.24 & 2.81±0.29 &   77.12±0.40 & 3.96±0.21 & 3.53±0.26  \\ 
            & w/o GD & 65.76±0.78 & 2.07±1.68 & 3.04±0.46   & 65.80±0.20 & 2.68±1.81 & 3.40±1.16 &   77.80±0.67 & 3.72±0.11 & 2.59±0.17 \\ 
            & w/o NST & 67.15±0.62 & 3.67±2.12 & 0.83±0.32  & 66.32±0.38 & 2.41±1.13 & 2.04±0.24 &   78.56±1.16 & 4.32±0.58 & 1.73±0.34 \\ 
            \bottomrule
        \end{tabular} 
    }
\end{table*}

\subsection{Ablation study}
To answer \textbf{RQ2}, we perform ablation experiments to systematically assess the impact of each component in FairDTD on its overall performance. Specifically, we study the influence of four components, namely the feature teacher, the structure teacher, the graph-level distillation, and the learning of node-specific temperatures, on learning fair GNNs. We adopt \textbf{w/o FT} to represent the exclusion of the feature teacher. We adopt \textbf{w/o ST} to represent the exclusion of the structure teacher. We adopt \textbf{w/o GD} to represent the exclusion of the graph-level distillation. We adopt \textbf{w/o NST} to represent the exclusion of the learning of node-specific temperatures and the adoption of a fixed temperature. Table \ref{ablation} presents the results for the different model variants.

From Table \ref{ablation}, the four variants exhibit inferior performance compared to FairDTD in balancing utility and fairness, proving each component's effectiveness and necessity. Compared with the fixed temperature, learning node-specific temperatures can help FairDTD learn more effective knowledge from the dual-teacher model. Moreover, graph-level distillation can help the dual-teacher model learn richer information in the multi-layer neural network to guide the fair GNN student. Finally, compared with \textbf{w/o FT} and \textbf{w/o ST}, FairDTD achieves better performance, which proves that dual-teacher distillation can promote knowledge transfer from the teacher to the student more effectively. It is worth noting that when comparing \textbf{w/o FT} with \textbf{w/o ST}, we observe that although both teachers contain fairness information, the feature teacher takes into account the model utility, while the structure teacher pays more attention to the fairness performance. Therefore, we need to better balance the influence of the feature teacher and the structure teacher's weights on the model to provide targeted guidance to the fair GNN student. Finally, by comparing the best results across the three datasets, we observe that the model’s performance declines regardless of which component is removed. This further demonstrates that the contributions of each component in FairDTD generalize across datasets, rather than being specific to any single dataset.

\begin{figure*}[htbp]
    \raggedright
    \
    \begin{subfigure}[b]{0.24\linewidth}
        \centering
        \includegraphics[width=1\linewidth]{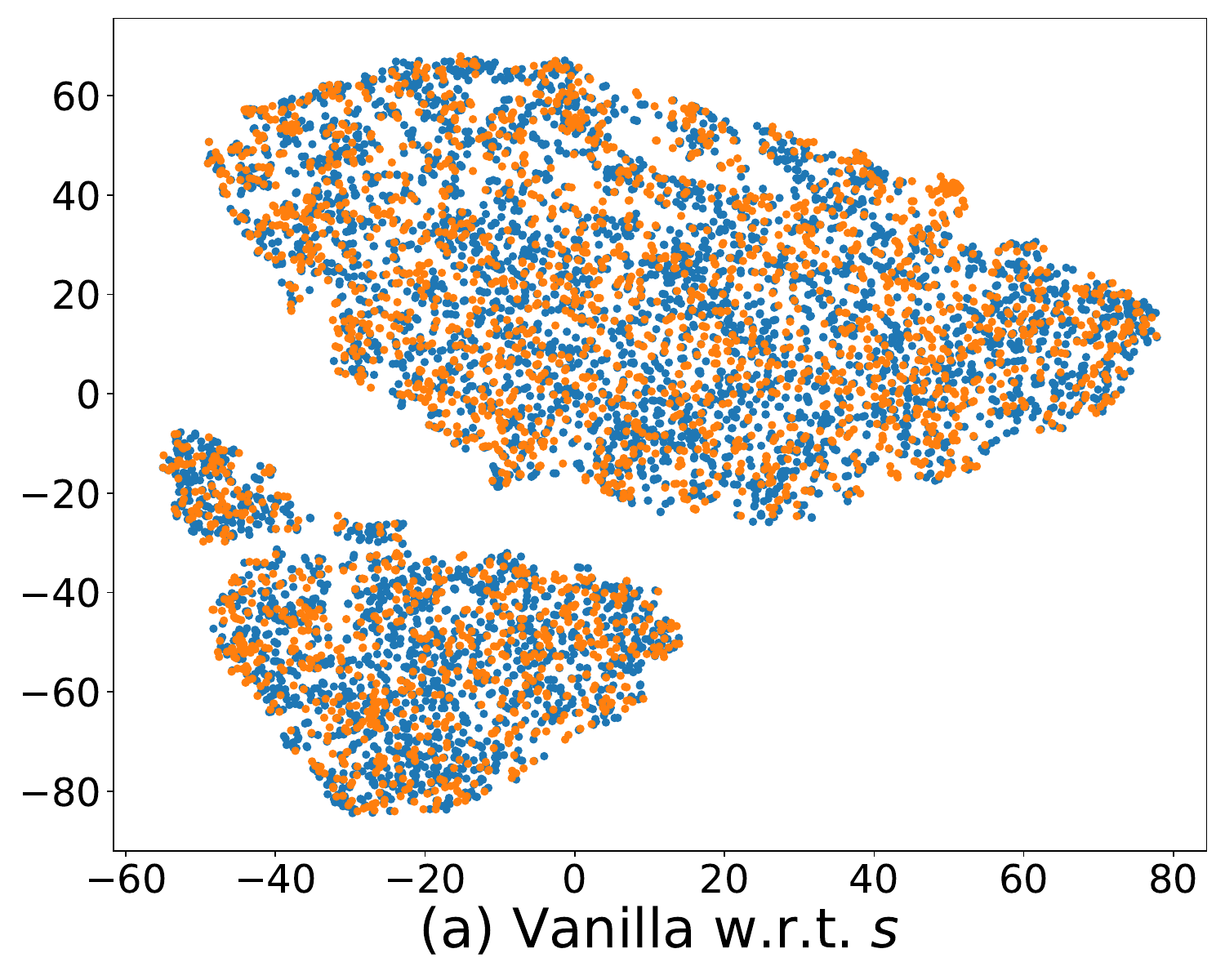}
    \end{subfigure}
    \begin{subfigure}[b]{0.24\linewidth}
        \centering
        \includegraphics[width=1\linewidth]{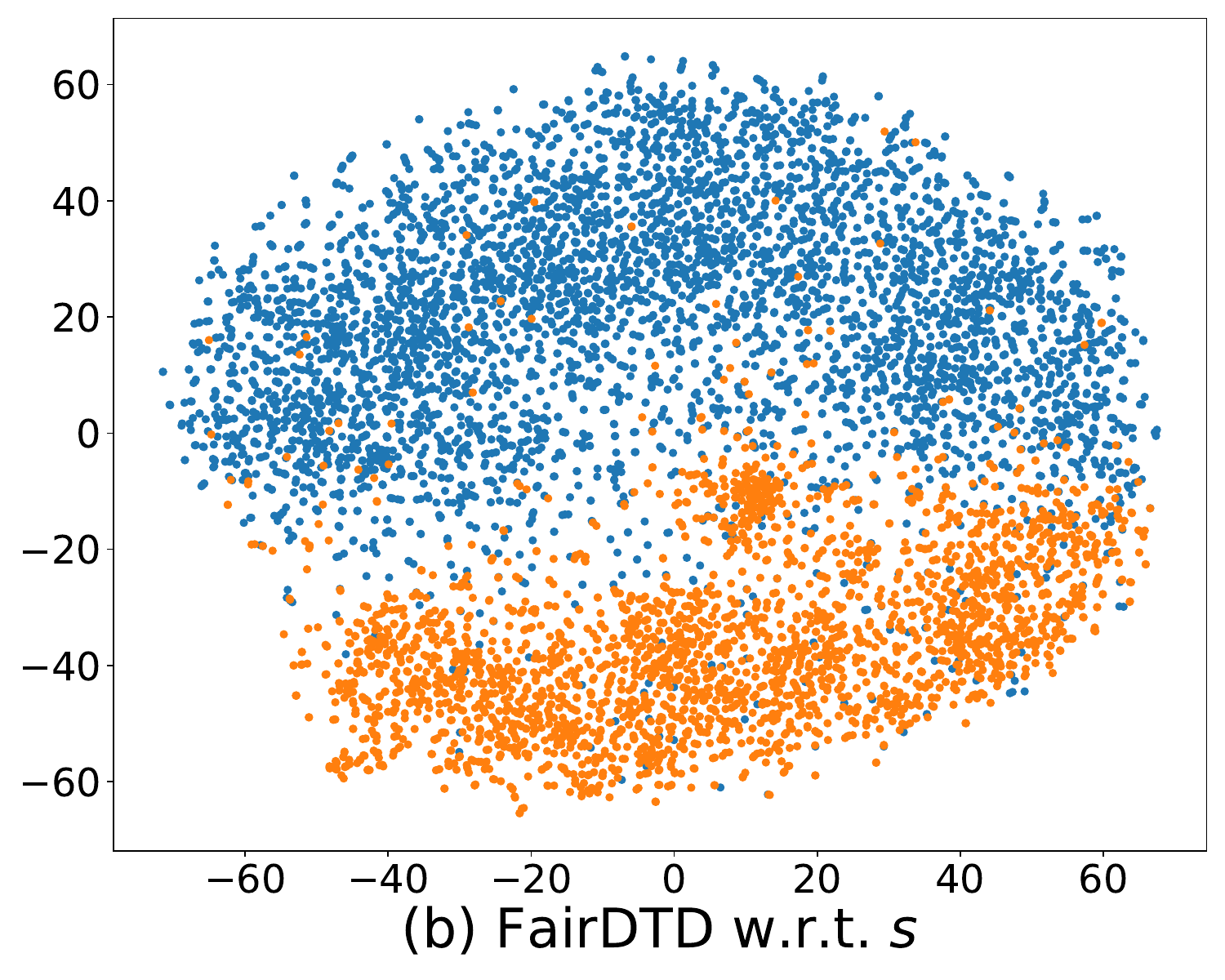}
    \end{subfigure}
    \begin{subfigure}[b]{0.24\linewidth}
        \centering
        \includegraphics[width=1\linewidth]{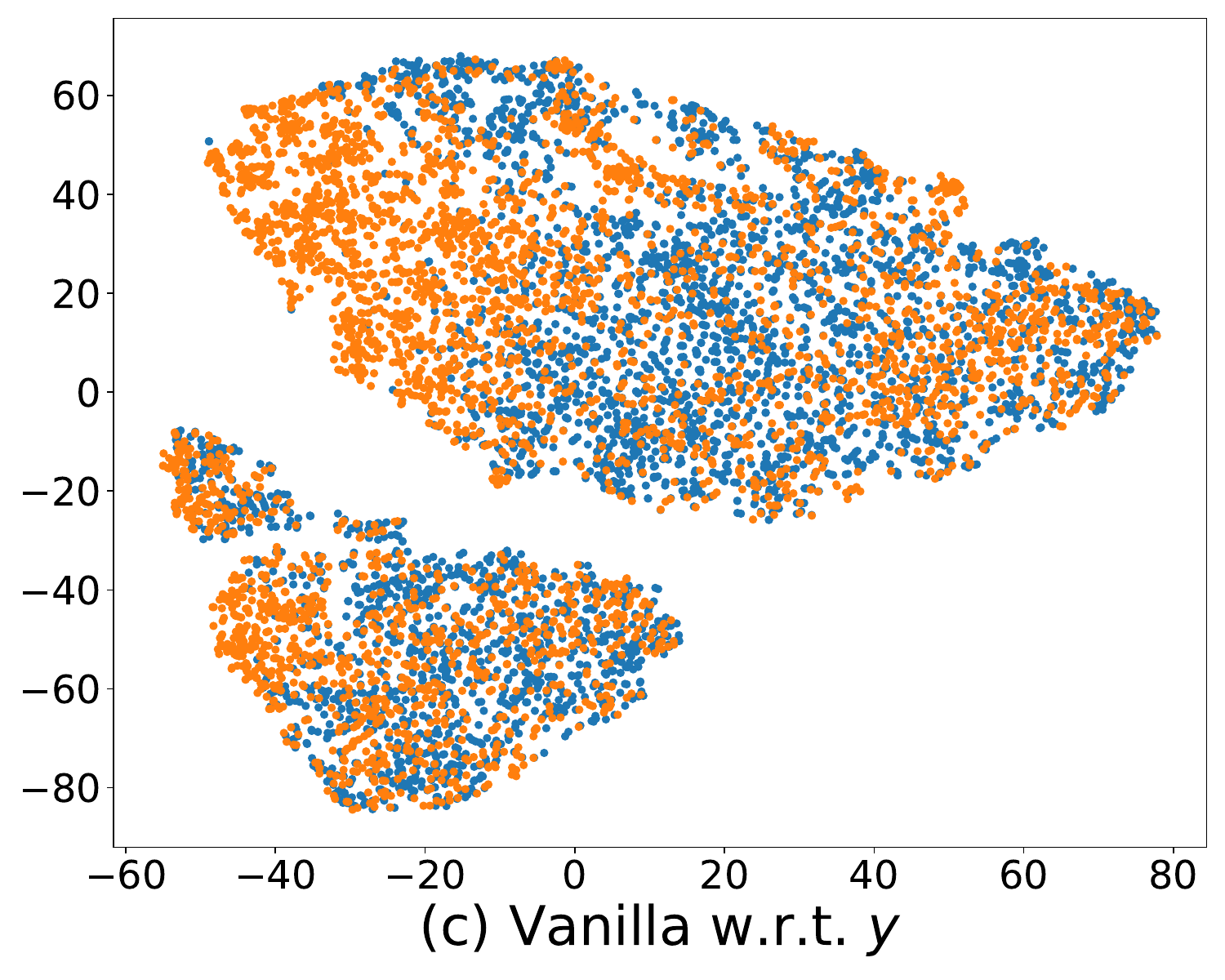}
    \end{subfigure}
    \begin{subfigure}[b]{0.24\linewidth}
        \centering
        \includegraphics[width=1\linewidth]{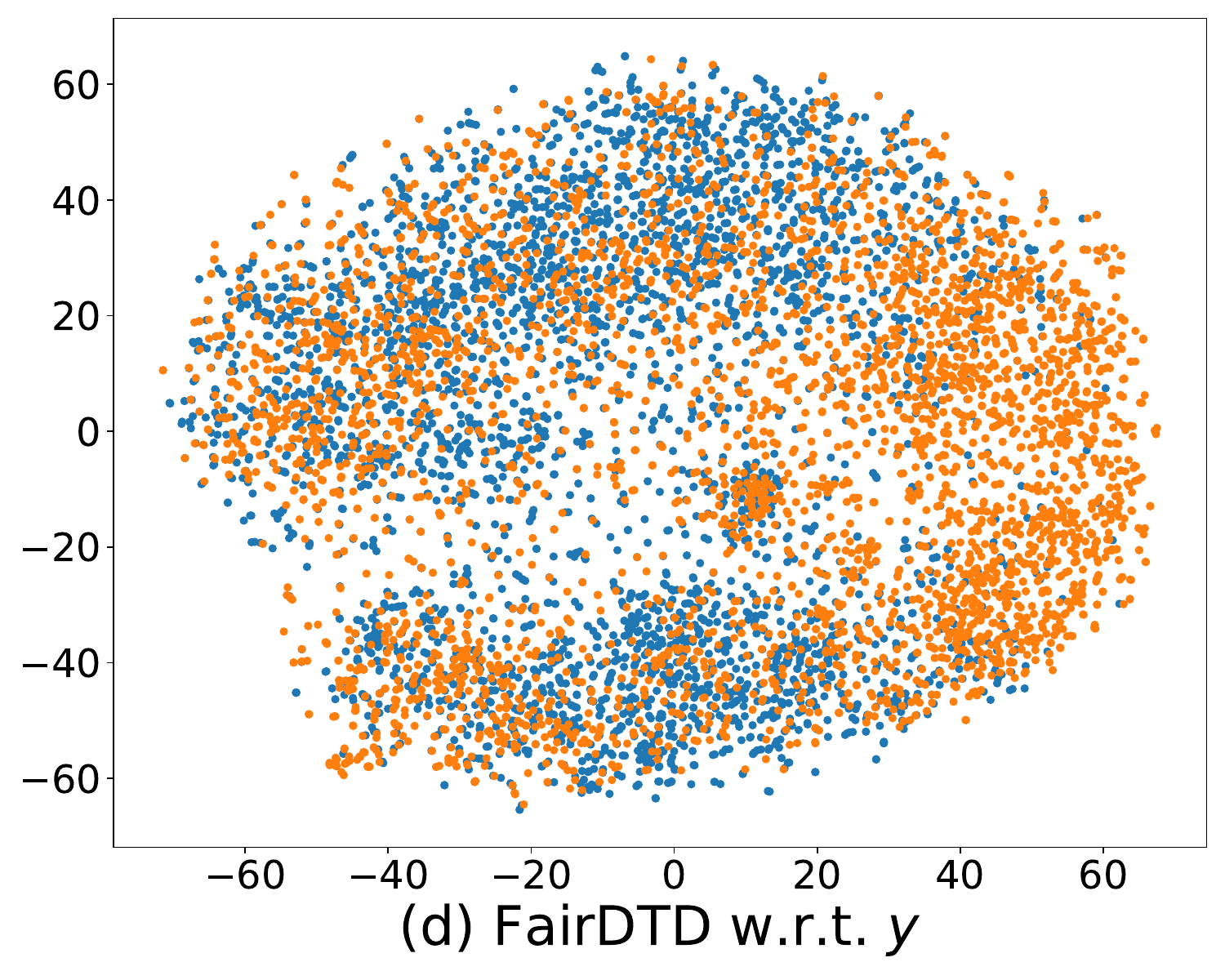}
    \end{subfigure}
    \captionsetup{justification=centering}
    \caption{\textbf{Visualizations of the node representations learned on the Pokec-z dataset.}}
    \label{visualization}
\end{figure*}

\begin{figure}[htbp] 
    \centering
    \begin{minipage}{0.49\linewidth}
        \centering
        \includegraphics[width=1\linewidth]{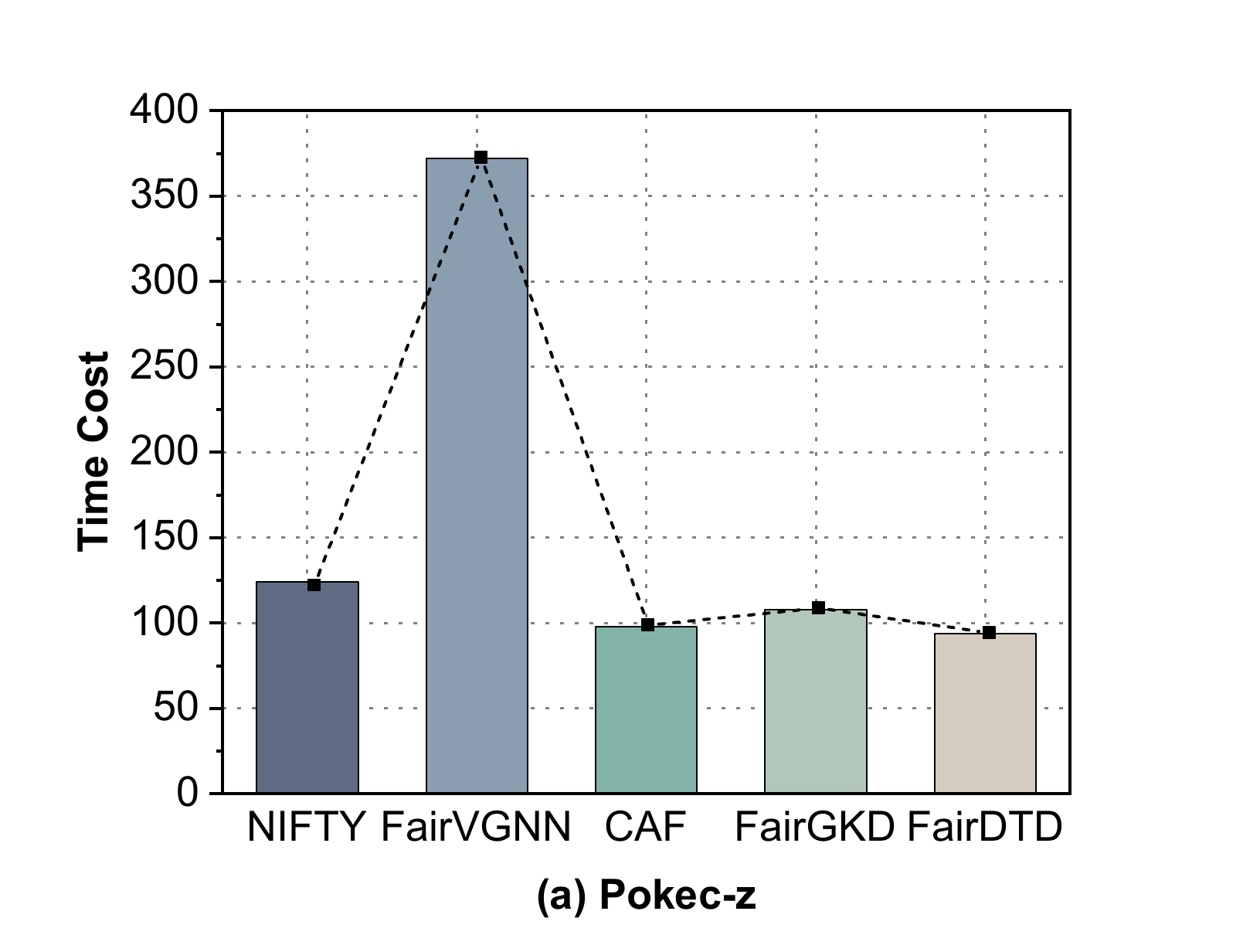} 
        
    \end{minipage}
    \hfill
    \begin{minipage}{0.49\linewidth}
        \centering
        \includegraphics[width=1\linewidth]{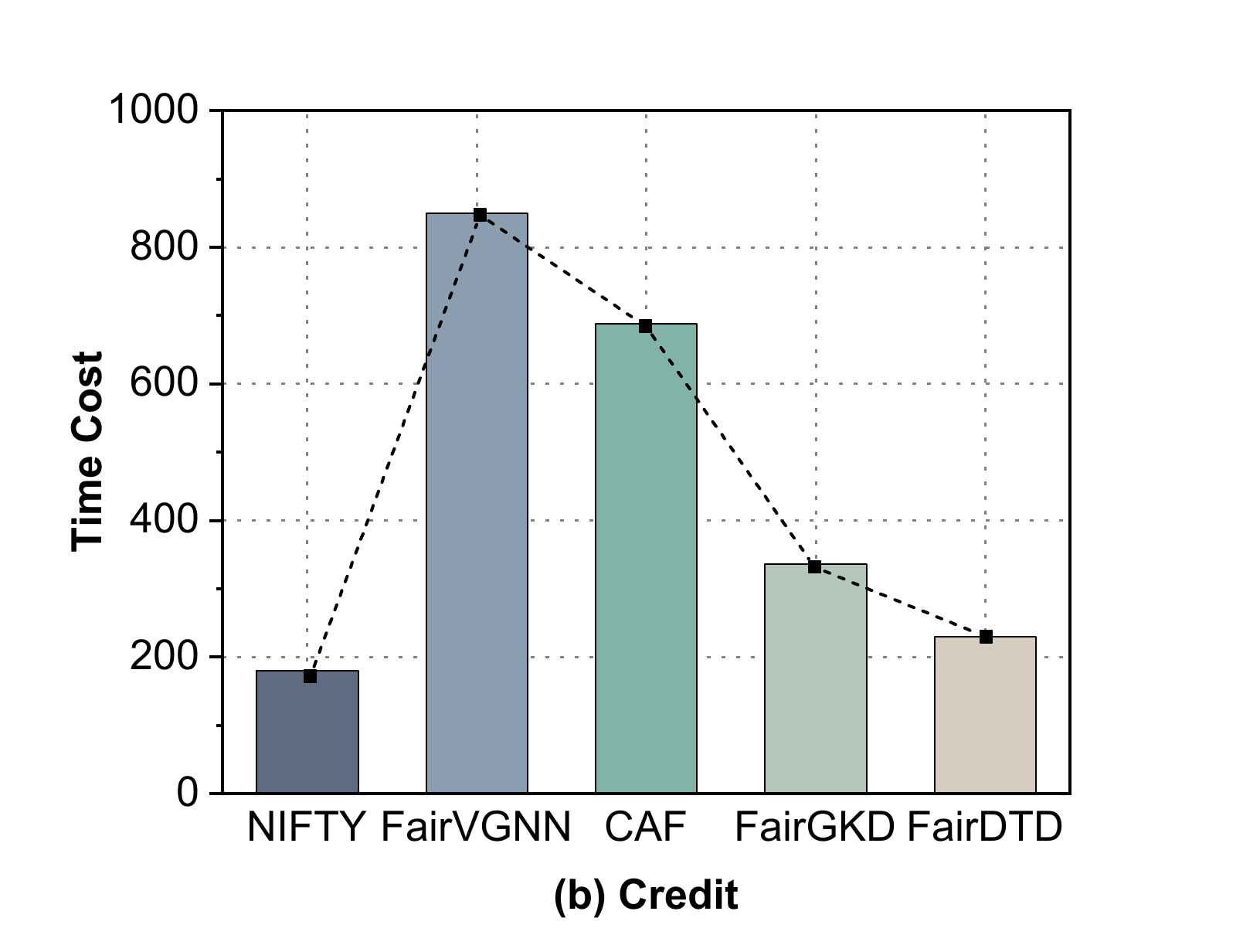} 
        
    \end{minipage}
    \caption{\textbf{Comparison of training time between baseline methods and FairDTD on Pokec-z and Credit (in seconds)}}
    \label{time}
\end{figure}

\subsection{Node representations visualization}
To answer \textbf{RQ3}, we visualize the node representations learned by FairDTD and the vanilla GNN model (i.e., GCN) on the Pokec-z dataset, as shown in Fig. \ref{visualization}. To obtain better visibility, we select samples from the test set. Specifically, the learned representations are projected into a two-dimensional space for visualization using t-SNE. We plot the node representations for both models concerning sensitive attributes and true labels.

First, comparing Fig. 4(a) and Fig. 4(b), we observe that FairDTD, guided by dual-teacher distillation, exhibits stronger discriminatory ability, resulting in clearer classification boundaries for nodes with different sensitive attributes. However, comparing Fig. 4(b) and Fig. 4(d), it is evident that in FairDTD, the classification boundary for sensitive attributes is orthogonal to that of the true labels, indicating that the model’s predictions for true labels do not rely on sensitive attributes \citep{li2024rethinking}. This proves the excellent fairness of FairDTD on Pokec-z. Then, comparing Fig. 4(c) and Fig. 4(d), we can observe that the discrimination of nodes with different true labels on FairDTD is further improved, demonstrating that FairDTD improves the model's utility. 

Additionally, it is evident that nodes with distinct true labels exhibit significant overlap in both Fig. 4(c) and Fig. 4(d). This is attributed to the fact that all models currently exhibit suboptimal predictive performance on the Pokec-z dataset, failing to differentiate these confounding node representations effectively. 

In summary, these results demonstrate FairDTD’s excellence in balancing fairness and model utility, achieving clear classification boundaries while ensuring predictions are not influenced by sensitive attributes.

\subsection{Training time comparison}
To answer \textbf{RQ4}, we analyze the training time overhead of FairDTD with that of the baseline methods on the Pokec-z and Credit datasets, as illustrated in Fig. \ref{time}. Overall, we can observe that FairDTD shows a lower computational cost among all methods, proving our method's efficiency. FairVGNN shows the highest computational cost among all methods because it encompasses a considerable number of parameters and requires a reverse training process. In addition, CAF performs well on Pokec-z but poorly on Credit. This might be related to its counterfactual augmentation module, which aims to find counterfactuals from the training data, resulting in a higher computational cost on larger datasets. Furthermore, compared to FairGKD, which is also based on knowledge distillation, FairDTD exhibits lower time overhead, further validating the efficiency and effectiveness of our proposed FairDTD method.

\begin{figure*}[htbp]
    \raggedright
    \
    \begin{subfigure}[b]{0.32\linewidth}
        \centering
        \includegraphics[width=1\linewidth]{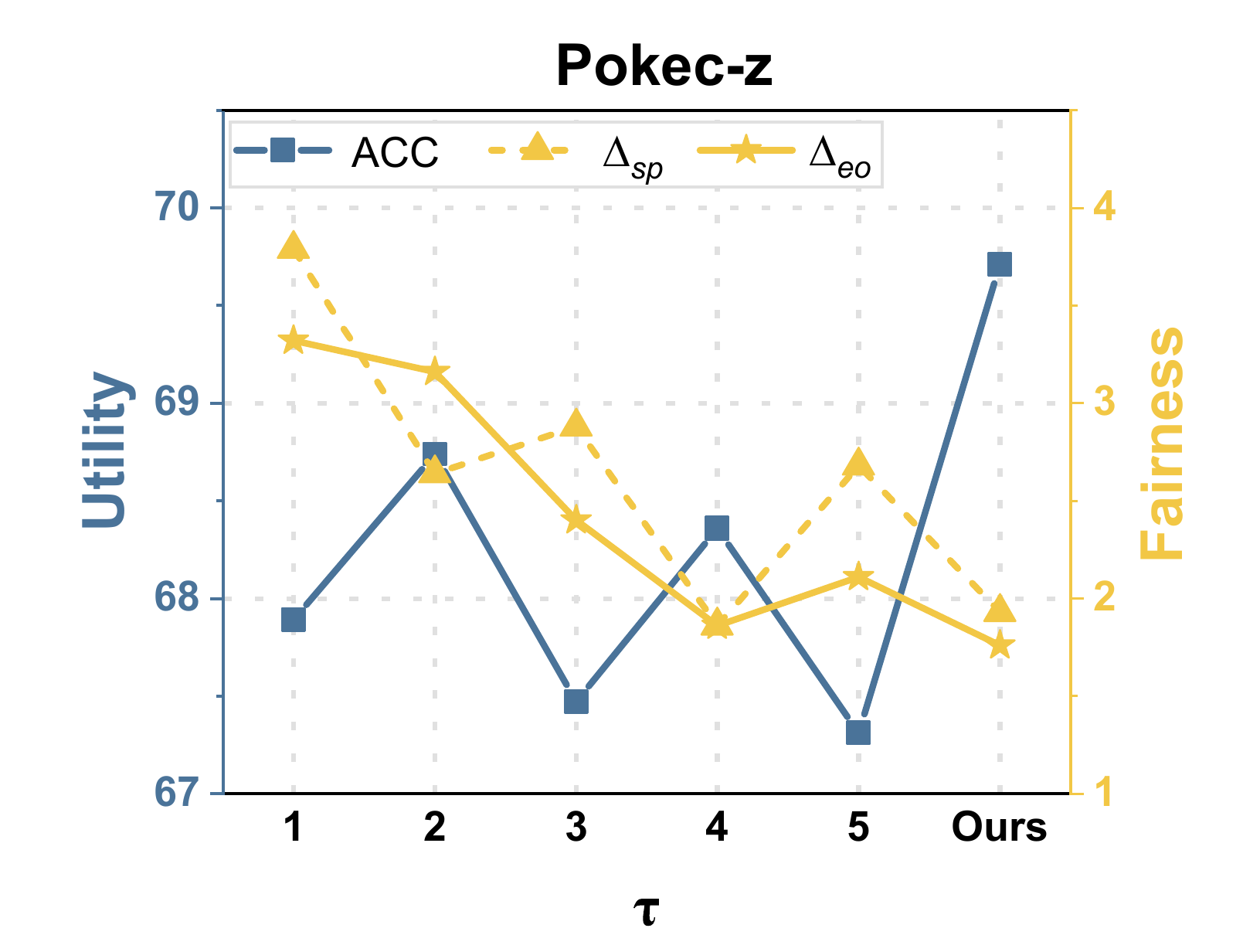}
    \end{subfigure}
    \begin{subfigure}[b]{0.32\linewidth}
        \centering
        \includegraphics[width=1\linewidth]{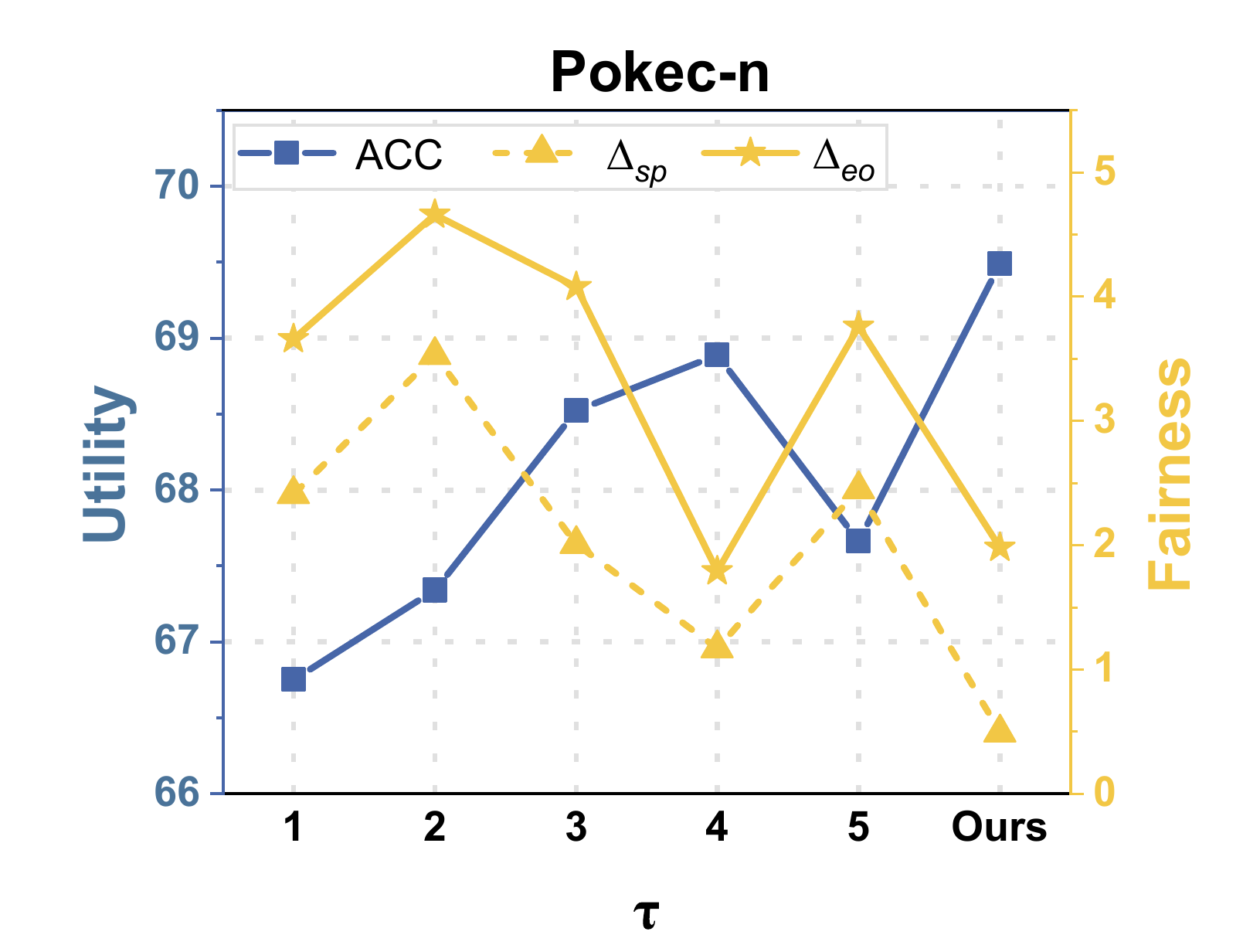}
    \end{subfigure}
    \begin{subfigure}[b]{0.32\linewidth}
        \centering
        \includegraphics[width=1\linewidth]{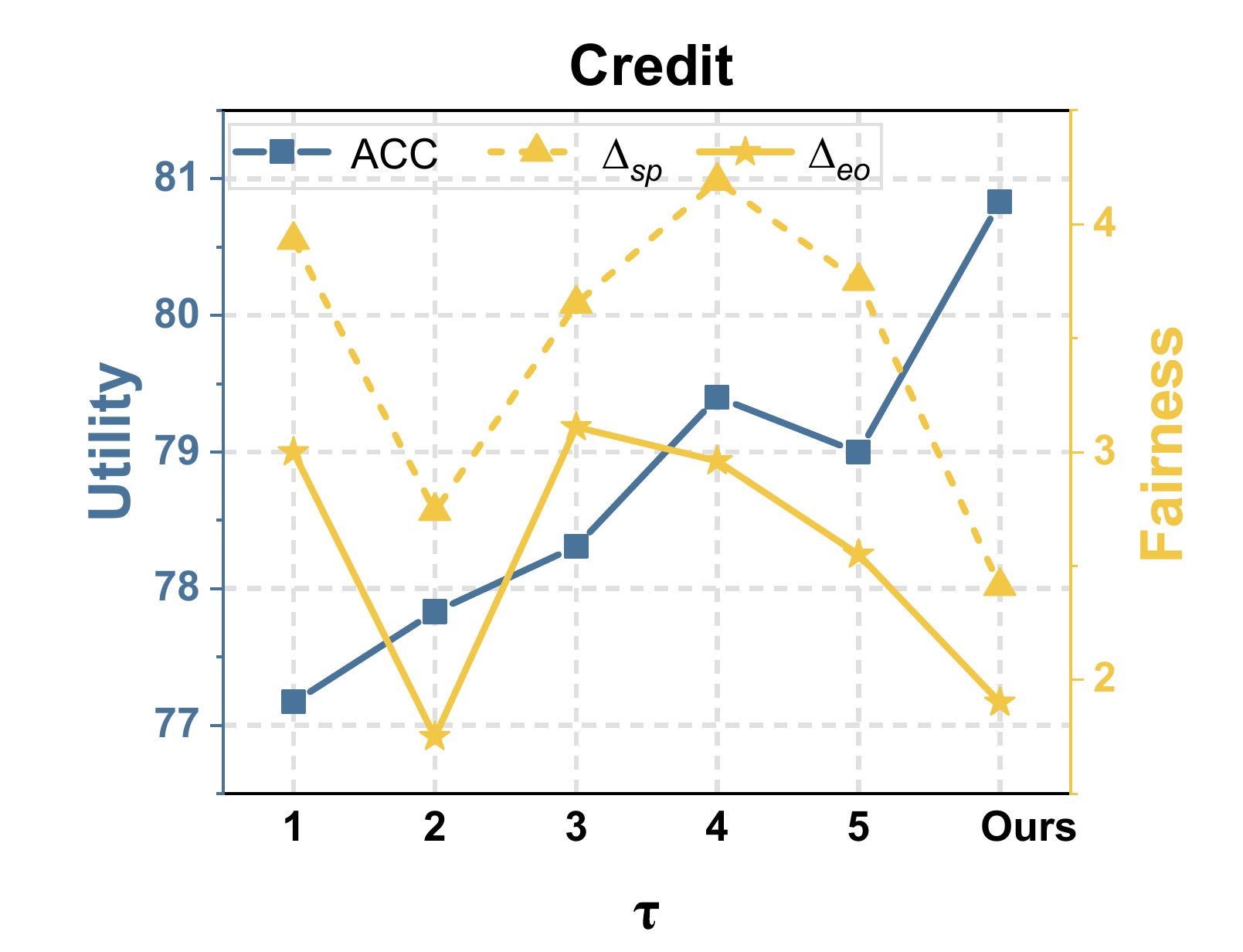}
    \end{subfigure}
    \captionsetup{justification=centering}
    \caption{\textbf{Results of FairDTD with node-specific temperature learning and fixed temperatures ranging from 1 to 5.}}
    \label{temperature}
\end{figure*}

\begin{figure}[htbp] 
    \centering
    \begin{minipage}{0.49\linewidth}
        \centering
        \includegraphics[width=1\linewidth]{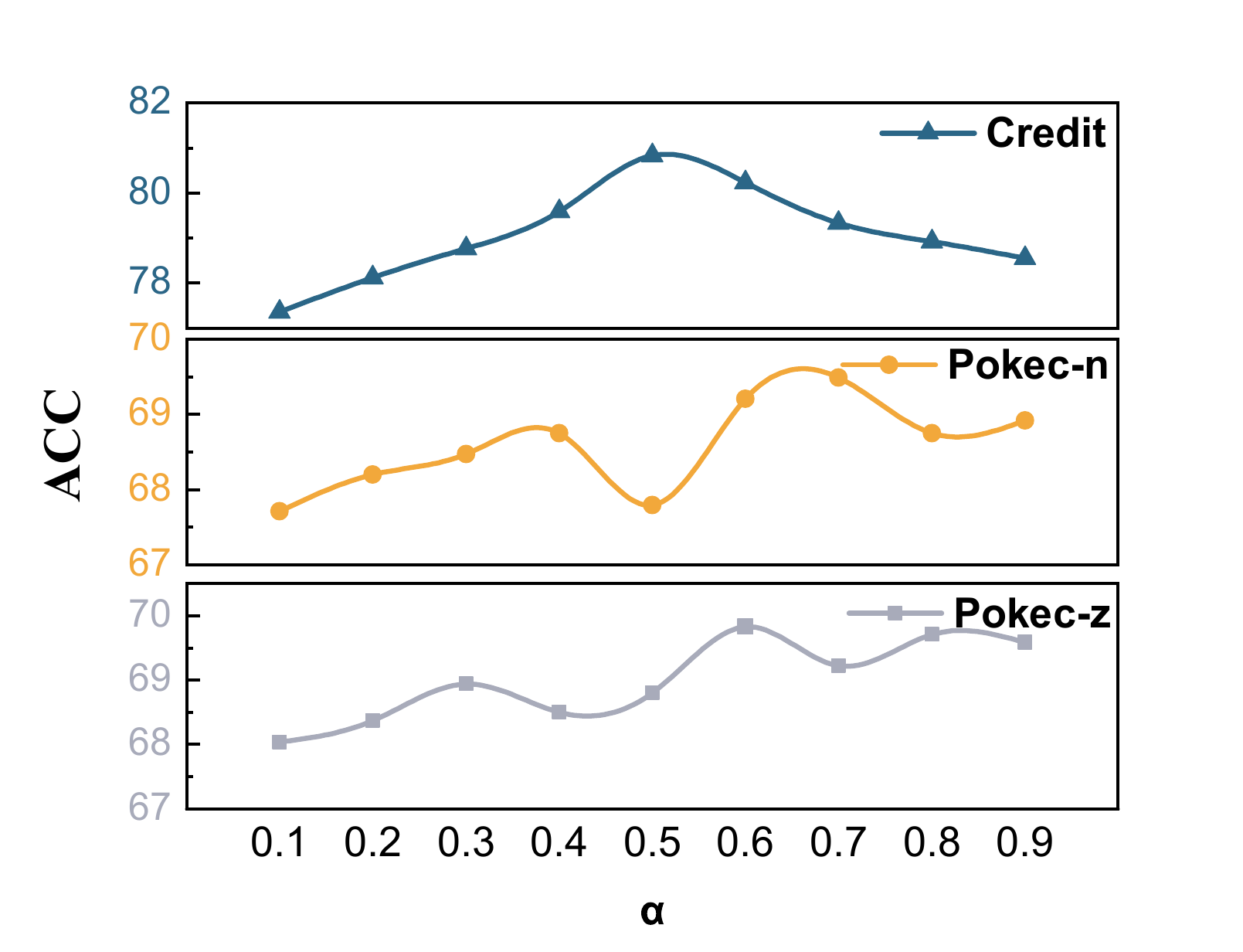} 
        
    \end{minipage}
    \hfill
    \begin{minipage}{0.49\linewidth}
        \centering
        \includegraphics[width=1\linewidth]{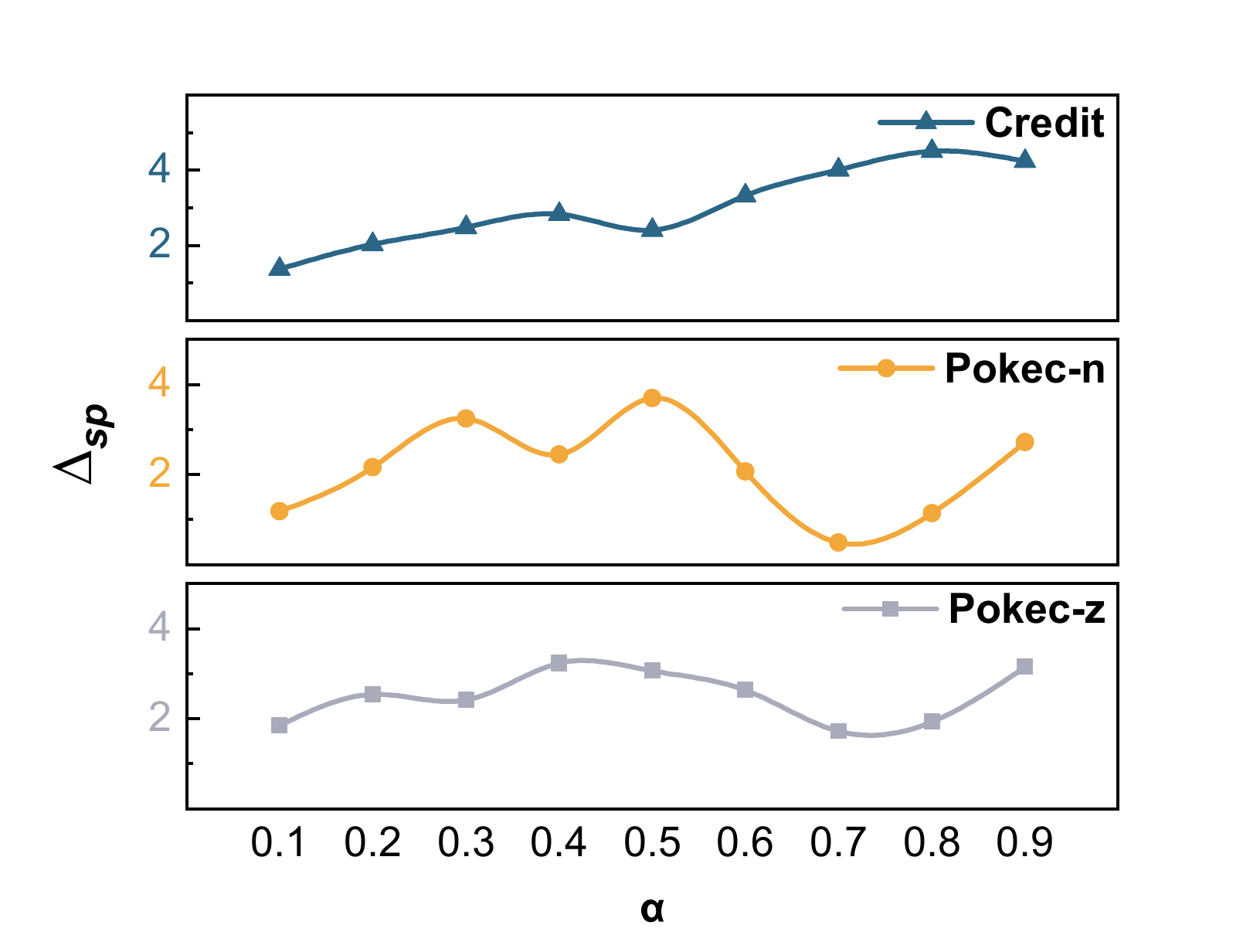} 
        
    \end{minipage}
    \caption{\textbf{Sensitivity analysis of the balance hyperparameter $\alpha$ on the three datasets.}}
    \label{alpha}
\end{figure}

\subsection{Hyper-parameter sensitivity analysis}
To answer \textbf{RQ5}, we performed a sensitivity analysis of the hyperparameters in FairDTD.

{\raggedright\textbf{Analysis of distillation temperature $\tau$.}}\hspace{1em}First, we further verify the influence of the temperature hyperparameter $\tau$ on the model. From Table \ref{ablation}, we have already confirmed the effectiveness of learning node-specific temperatures. Here, we further compare learning node-specific temperatures with fixed temperatures ranging from 1 to 5. Fig. \ref{temperature} shows that learning node-specific temperatures can achieve the best balance between model utility and fairness, indicating that learning node-specific temperatures can effectively distinguish between confounding information and true information to facilitate knowledge transfer.

{\raggedright\textbf{Analysis of balance parameter $\alpha$.}}\hspace{1em}Next, we analyzed the sensitivity of the balance parameter $\alpha$, which determines the relative weights of the feature and structure teachers. To assess its influence on FairDTD, we varied $\alpha$ from 0.1 to 0.9 in increments of 0.1 and recorded the experimental results. As shown in Fig. \ref{alpha}, increasing $\alpha$ improves model accuracy but reduces fairness. This observation aligns with Fig. \ref{partial}, where the feature teacher exhibits higher accuracy while the structure teacher prioritizes fairness. Consequently, the selection of an appropriate $\alpha$ is essential to strike an optimal balance between utility and fairness.

\section{Discussion}
From an applied perspective, FairDTD demonstrates significant practical value across various real-world scenarios. In the social network scenario, evaluations on the Pokec-z and Pokec-n datasets indicate that FairDTD achieves group fairness with respect to the region, significantly reducing prediction disparities among users from different regions. In the credit scoring scenario, experiments on the Credit dataset show that FairDTD effectively mitigates decision-making bias caused by age, thereby enhancing model fairness and transparency. Furthermore,  in socially sensitive graph-structured application scenarios such as healthcare and public governance, it is essential to adopt FairDTD for fair representation learning to prevent traditional GNNs from inadvertently amplifying group biases.

Although FairDTD performs exceptionally well in terms of fairness, it does have some limitations. Specifically, when dealing with changes in sensitive attributes, FairDTD may require retraining to achieve optimal performance. In future work, we plan to explore integrating FairDTD with invariant learning as a potential solution to this issue. Additionally, FairDTD currently focuses on a single sensitive attribute; future research will aim to extend the framework to scenarios involving multiple sensitive attributes.

\section{Conclusion}
In this paper, we present a novel perspective for learning fair GNNs. By constructing causal structure models, we identify that bias originates from both node features and graph structures. Through our analysis, we find that a simple partial data training strategy, where the model is trained exclusively on node features or graph structures, can achieve fairness that is on par with, or potentially exceeds, the performance of current state-of-the-art fair GNN methods, albeit at the expense of model utility. To address this trade-off, we propose FairDTD, a framework that combines knowledge distillation with partial data training to learn fair GNNs. Specifically, FairDTD employs two fairness-oriented teacher models, referred to as the feature expert and the structure expert, to guide the learning of a fair GNN student. Additionally, it incorporates graph-level distillation and node-specific temperature learning to facilitate the transfer of knowledge, improving both fairness and model utility.
Experimental results on three real-world datasets demonstrate the effectiveness of FairDTD in balancing fairness and model utility.

\section*{CRediT authorship contribution statement}
\textbf{Chengyu Li}: Writing - original draft, Software, Methodology, Conceptualization, Visualization. \textbf{Debo Cheng}: Conceptualization, Writing - review \& editing, Formal analysis. \textbf{Guixian Zhang}: Conceptualization, Writing - review \& editing, Formal analysis. \textbf{Yi Li}: Writing - review \& editing, Formal analysis. \textbf{Shichao Zhang}: Writing - review \& editing, Supervision, Funding acquisition.

\section*{Declaration of competing interest}
The authors declare that they have no known competing financial interests or personal relationships that could have appeared to influence the work reported in this paper.

\section*{Acknowledgements}
This work has been supported in part by the Project of Guangxi Science and Technology (GuiKeAB23026040), the Research Fund of Guangxi Key Lab of Multi-source Information Mining \& Security (Nos. 20-A-01-01 and MIMS21-M01), the Guangxi Collaborative Innovation Center of Multi-Source Information Integration and Intelligent Processing, the Research Fund of Guangxi Key Lab of Multi-source Information Mining \& Security (MIMS24-03), the Research Fund of Guangxi Key Lab of Multi-source Information Mining \& Security (MIMS24-13).

\printcredits



\bibliographystyle{unsrt}
\bibliography{cas-refs}

\end{document}